\documentclass[11pt]{article}

\usepackage{multirow}
\usepackage{amssymb,amsmath,amsthm}
\usepackage{color}
\usepackage{graphicx}
\usepackage{tikz}
\usepackage{epstopdf}
\usepackage{romannum}
\usepackage{algorithm}
\usepackage{subcaption}
\usepackage{graphicx}
\usepackage{float}
\usepackage{multicol}
\usepackage[noend]{algpseudocode}
\usepackage{romannum}

\definecolor{verm}{rgb}{0.6,0.2,0.2}
\definecolor{purp}{rgb}{0.3,0.1,0.6}
\definecolor{purple}{rgb}{0.4,0.0,0.6}
\definecolor{bggreen}{rgb}{0.1,0.3,0.1}
\definecolor{dgreen}{rgb}{0.1,0.6,0.1}
\definecolor{black}{rgb}{0.0,0.0,0.0}
\definecolor{crim}{rgb}{0.3,0.1,0.1}
\definecolor{dred}{rgb}{0.5,0.1,0.1}

\definecolor{Blue}{cmyk}{0.65,0.13,0,0}
\definecolor{Black}{cmyk}{0,0,0,1}
\definecolor{Red}{cmyk}{0,1,1,0}
\definecolor{Green}{cmyk}{1,0,1,0}
\definecolor{Orange}{cmyk}{0,0.61,0.87,0.1}
\definecolor{Fuchsia}{cmyk}{0.47,0.91,0,0.08}
\definecolor{PineGreen}{cmyk}{0.92,0,0.59,0.25}

\def\1{{\bf 1}}

\def\C{{\cal C}}

\def\N{{\cal N}}

\def\Cbb{{\mathbb C}}
\def\Fbb{{\mathbb F}}

\def\R{{\mathbb R}}

\def\al{\alpha}
\def\d{\delta}
\def\D{\Delta}
\def\e{\epsilon}

\def\l{\lambda}

\def\OM{\Omega}
\def\r{\rho}
\def\s{\sigma}
\def\SI{\Sigma}
\def\t{\tau}
\def\th{\theta}

\def\db{\bar{d}}

\def\kb{\bar{k}}

\def\ap{\rightarrow}

\def\seq{\subseteq}

\def\bi{\{0,1\}}

\def\bimn{\bi^{m \times n}}

\def\bz{{\bf 0}}

\def\fa{\; \forall}

\def\st{\mbox{ s.t. }}

\def\nm{\Vert}

\renewcommand{\and}{\mbox{$\wedge$}}

\newcommand{\bc}{\begin{center}}
\newcommand{\ec}{\end{center}}
\newcommand{\be}{\begin{equation}}
\newcommand{\ee}{\end{equation}}
\newcommand{\bd}{\begin{displaymath}}
\newcommand{\ed}{\end{displaymath}}
\newcommand{\ba}{\begin{array}}
\newcommand{\ea}{\end{array}}
\newcommand{\ben}{\begin{enumerate}}
\newcommand{\een}{\end{enumerate}}
\newcommand{\bit}{\begin{itemize}}
\newcommand{\eit}{\end{itemize}}
\newcommand{\beq}{\begin{eqnarray}}
\newcommand{\eeq}{\end{eqnarray}}
\newcommand{\btab}{\begin{tabular}}
\newcommand{\etab}{\end{tabular}}
\newcommand{\bfig}{\begin{figure}}
\newcommand{\efig}{\end{figure}}
\newcommand{\btp}{\begin{tikzpicture}}
\newcommand{\etp}{\end{tikzpicture}}
\newcommand{\bcm}{\begin{comment}}
\newcommand{\ecm}{\end{comment}}

\newcommand{\argmin}{\operatornamewithlimits{arg min}}

\newcommand{\nmm}[1]{ \nm #1 \nm }
\newcommand{\nmeu}[1]{ \nm #1 \nm_2 }
\newcommand{\nmeusq}[1]{ \nm #1 \nm_2^2 }

\newcommand{\nmp}[1]{ \nm #1 \nm_p }
\newcommand{\nmq}[1]{ \nm #1 \nm_q }

\newcommand{\IP}[2]{ \langle #1 , #2 \rangle }

\def\DBP{\D_{{\rm BP}}}

\def\Cn{\Cbb^n}

\def\Rmn{\R^{m \times n}}

\def\xh{\hat{x}}

\def\nmsl1{\nm_{{\rm SL1}}}

\def\DBP{\D_{{\rm BP}}}
\def\db{\bar{d}}

{\bf}{\it}
\newtheorem{definition}{Definition}{\bf}{\it}
\newtheorem{example}{Example}{\bf}{\rm}
\newtheorem{lemma}{Lemma}{\bf}{\it}
\newtheorem{theorem}{Theorem}{\bf}{\it}
{\bf}{\it}
{\bf}{\it}
{\bf}{\rm}

\begin{document}

% \IEEEoverridecommandlockouts

\title{
Compressed Sensing Using Binary Matrices of Nearly Optimal Dimensions
}

\author{Mahsa Lotfi and Mathukumalli Vidyasagar
\thanks{ML is with the with the Department of Statistics,
Stanford University.
MV is with the Indian Institute of Technology Hyderabad.
This work was carried out when both authors were with the
the Erik Jonsson School of Engineering and Computer Science,
The University of Texas at Dallas, Richardson, TX 75080, USA,
and forms a part of the first author's doctoral thesis.
% and
% The University of Texas at Dallas.
Emails: lotfi@stanford.edu; m.vidyasagar@iith.ac.in.
% m.vidyasagar@utdallas.edu.
This research was supported by the National Science Foundation, USA under
Award \#ECCS-1306630, and by the Department of Science and Technology,
Government of India.
}
}

\maketitle
\thispagestyle{empty}
\pagestyle{empty}

%%%%%%%%%%%%%%%%%%%%%%%%%%%%%%%%%%%%%%%%%%%%%%%%%%%%%%%%%%%%%%%%%%%%%%%%%%%%%%%%
\begin{abstract}
In this paper, we study the problem of compressed sensing using binary
measurement matrices and $\ell_1$-norm minimization (basis pursuit)
as the recovery algorithm.
We derive new upper and lower bounds on the number of
measurements to achieve robust sparse recovery with binary matrices.
We establish sufficient conditions for a column-regular binary matrix
to satisfy the robust null space property (RNSP)
and show that the
associated sufficient conditions
% sparsity bounds
for robust sparse recovery obtained using the RNSP
are better by a factor of $(3 \sqrt{3})/2 \approx 2.6$ compared to
the sufficient conditions obtained using
the restricted isometry property (RIP).
Next we derive universal \textit{lower} bounds on the number of measurements
that any binary matrix needs to have in order to satisfy the weaker sufficient
condition based on the RNSP and show that bipartite graphs of girth six
are optimal.
Then we display two classes of binary matrices,
namely parity check matrices of array codes and Euler squares,
which have girth six and are nearly
optimal in the sense of almost satisfying the lower bound.
In principle, randomly generated Gaussian measurement matrices are 
``order-optimal.''
So we compare the phase transition behavior of the basis pursuit formulation
using binary array codes
and Gaussian matrices and show that (i) there is essentially
no difference between the phase transition boundaries in the two cases and
(ii) the CPU time of basis pursuit with binary matrices is hundreds of times
faster than with Gaussian matrices and the storage requirements are less.
Therefore it is suggested that binary matrices are a viable alternative to
Gaussian matrices for compressed sensing using basis pursuit.
\end{abstract}

\section{Introduction}\label{sec:Intro}

Compressed sensing refers to the recovery of high-dimensional but
low-complexity entities from a limited number of measurements.
The specific problem studied in this paper is to recover a vector $x \in \R^n$,
where only $k \ll n$ components are significant and the rest are either
zero or small,
based on a set of linear measurements $y = Ax$, where $A \in \Rmn$.
A variant is when $y = Ax + \eta$, where $\eta$ denotes measurement noise
and a prior bound of the form $\nmm{\eta} \leq \e$ is available.
By far the most popular solution methodology for this problem is
\textit{basis pursuit} in which an
approximation $\xh$ to the unknown vector $x$ is constructed via
\be\label{eq:01}
\xh := \argmin_z \nmm{z}_1 \st \nmm{y - Az} \leq \e .
\ee
The basis pursuit approach (with $\eta = 0$ so that the constraint
in \eqref{eq:01} becomes $y = Az$) was proposed in
\cite{Chen-Donoho-Saunders98,Chen-Donoho-Saunders01}, but
without guarantees on its performance.
Much of the subsequent research in compressed sensing
has been focused on the case where $A$
consists of $mn$ independent samples of a zero-mean, unit-variance
Gaussian or sub-Gaussian random variable, normalized by $1/\sqrt{m}$.
With this choice, it is shown in \cite{Candes-Tao05} that,
with high probability with respect to the process of generating $A$,
$m = O(k \ln (n/k))$ measurements suffice to ensure that
$\xh$ defined in \eqref{eq:01} equals $x$, provided $x$ is sufficiently sparse.
It is also known that \textit{any} compressed sensing algorithm
requires $m = \OM(k \ln(n/k))$ samples; see
\cite{Cohen-Dahmen-DeVore09} for an early result and
\cite{BIPW10} for a simpler and more explicit version of this bound.
Thus random Gaussian matrices are ``order optimal'' in the sense that
the number of measurements is within a fixed universal constant of the
minimum required.

In recent times, there has been a lot of interest in the
use of \textit{sparse binary} measurement matrices for compressed sensing.
One of the main advantages of this approach
is that it allows one to connect compressed sensing to fields such as
graph theory and algebraic coding theory.
There are also some computational advantages.
At present, a popular alternative is to choose the measurement matrix $A$ to
consist of $mn$ independent samples of a Gaussian random variable.
A Gaussian random variable is \textit{nonzero} with probability one;
therefore every element of $A$ will be nonzero with probability one.
Moreover, in solving the minimization problem in \eqref{eq:01},
each element of $A$ needs to be stored to high precision.
In contrast, sparse binary matrices require less storage both because
they are sparse and also because every nonzero element equals one.
For this reason, binary matrices are also said to be ``multiplication-free.''
As a result, popular compressed sensing approaches such as
\eqref{eq:01} can be applied effectively
for far larger values of $m$ and $n$
and with greatly reduced CPU time, when $A$ is a sparse
binary matrix instead of a random Gaussian matrix.
Of course, the previous discussion assumes that the unknown vector is
sparse in the canonical basis.
There are situations, where the unknown vector is sparse with respect to
some other basis, such as the Fourier basis.
Our remarks would not apply in such a situation.

At present, the best available bounds for the number of measurements
required by a binary matrix are $m = O(\max \{ k^2 , \sqrt{n} \})$.
This contrasts with $m = O(k \ln (n/k))$ for random Gaussian matrices.
However, in the latter case, the $O$ symbol hides a very large constant.
It is shown in this paper that for
values of $n \lesssim 10^5$, the \textit{known} bounds with
binary matrices are in fact \textit{smaller} than with random Gaussian matrices.
The preceding discussion refers to the case where a particular
matrix $A$ is \textit{guaranteed} to recover
\textit{all} sufficiently sparse vectors.
A parallel approach is to study conditions under which ``most''
sparse vectors are recovered.
Specifically, in this approach, $n,m$ are fixed and $k$ is varied
from $1$ to $m$.
For each choice of $k$, a large number of vectors with exactly $k$
nonzero components are generated at random and the fraction that
is recovered accurately is computed.
Clearly, as $k$ is increased, this fraction decreases.
% But the phenomenon of interest is known as ``phase transition.''
One might expect that the fraction of recovered randomly generated vectors 
equals $1$ when $k$ is sufficiently small and decreases gradually to $0$
as $k$ approaches $m$.
In reality there is a \textit{sharp boundary} below which almost all $k$-sparse
vectors are recovered and above which almost no $k$-sparse vectors
are recovered.
This phenomenon is known as \textit{phase transition} and
has been established theoretically for the case, where $A$ consists of
random samples from a Gaussian distribution in
\cite{Donoho06b,Donoho-Tanner-PNAS05,Donoho-Tanner-JAMS09}.
A very general theory is derived in \cite{Amelunxen-et-al14},
where the measurement matrix still consists of random Gaussians,
but the objective function is changed from the $\ell_1$-norm to
an arbitrary convex function.
In a recent paper \cite{MJGD13},
phase transitions are studied \textit{empirically}
for several classes of \textit{deterministic} measurement matrices
and it is verified that there is essentially no difference between the phase transitions of
of deterministic
  measurement matrices and
  the phase transitions of
  random Gaussian
  measurement matrices.
% the phase transitions with random Gaussian matrices.

Here we describe the organization of the paper, as well as its contributions.
Section \ref{sec:Back} contains background material,
but also includes some improvements over known results.
%The relationship between the two is discussed in Section \ref{sec:rel}.}
%%
In particular, we review the current
literature on the construction of binary matrices
for compressed sensing.
The original contributions of the paper begin with Section \ref{sec:RNSP}.
In this section we derive a sufficient condition for a binary matrix to
satisfy the robust null space property (RNSP).
In turn this leads to a new upper bound on the sparsity count $k$ for
which robust sparse recovery can be guaranteed using a column-regular
binary matrix.\footnote{This term is defined in Section \ref{sec:RNSP}.}
In Section \ref{sec:Lower} we derive a lower bound on the number of
measurements $m$ as a function of the girth of the bipartite graph associated
with the measurement matrix; it is shown that graphs of girth six are optimal
in terms of minimizing the number of measurements.
In Section \ref{sec:Girth}, we construct binary matrices of girth six,
where the number of measurements is nearly equal to the lower bound derived
in Section \ref{sec:Lower}; this explains the title of the paper.
In Section \ref{sec:Low}, we attempt to reconcile two seemingly conflicting
observations, namely: For compressed sensing, graphs of girth six are
optimal, whereas in coding, graphs of high girth are preferred.
In Section \ref{sec:num}, we carry out some numerical experiments and 
establish that the basis pursuit approach together with our binary matrices
exhibits a phase transition.
The paper is concluded with some discussion in Section \ref{sec:Disc}.

\section{Background}\label{sec:Back}

\subsection{Definition of Compressed Sensing}\label{ssec:prob}
Let $\SI_k \seq \R^n$ denote the set of $k$-sparse vectors in $\R^n$; i.e.,
\bd
\SI_k := \{ x \in \R^n : \nmm{x}_0 \leq k \} ,
\ed
where, as is customary, $\nmm{\cdot}_0$ denotes the number of nonzero
components of $x$.
Given a norm $\nmm{\cdot}$ on $\R^n$, the \textbf{$k$-sparsity index}
of $x$ with respect to that norm is defined by
\bd
\s_k(x, \nmm{\cdot}) := \min_{z \in \SI_k} \nmm{x-z} .
\ed
Now we are in a position to define the compressed sensing problem precisely.
Note that $A \in \Rmn$ is called the measurement matrix and 
$\D: \R^m \ap \R^n$ is called the ``decoder map.''

\begin{definition}\label{def:vec-rec}
The pair $(A,\D)$ is said to achieve \textbf{stable sparse
recovery of order $k$} and indices $p,q$
if there exists a constant $C$ such that
\be\label{eq:12}
\nmp{ \D(Ax) - x} \leq C \s_k( x , \nmq{\cdot} ) , \fa x \in \R^n .
\ee
The pair $(A,\D)$ is said to achieve \textbf{robust sparse
recovery of order $k$} and indices $p,q$ (and norm $\nmm{\cdot}$)
if there exist constants $C$ and $D$ such that,
for all $\eta \in \R^m$ with $\nmm{\eta} \leq \e$, it is the case that
\be\label{eq:13}
\nmp{ \D (Ax + \eta) - x} \leq C \s_k( x , \nmq{\cdot} )
+ D \e , \fa x \in \Cn .
\ee
\end{definition}

The above definitions apply to general norms.
In this paper and indeed in much of the compressed sensing literature,
the emphasis is on the case, where $q = 1$ and $p \in [1,2]$.
However, the norm on $\eta$ is still arbitrary.

\subsection{Approaches to Compressed Sensing -- I: RIP}\label{ssec:RIP}

Next we present some sufficient conditions for basis pursuit as defined in
\eqref{eq:01} to achieve robust or stable sparse recovery.
There are two widely used sufficient conditions, namely the restricted
isometry property (RIP) and the stable (or robust) null space property
(SNSP or RNSP).
We begin by discussing the RIP.

\begin{definition}\label{def:RIP}
A matrix $A \in \Rmn$ is said to satisfy the \textbf{restricted isometry
property (RIP)} of order $k$ with constant $\d$ if
\be\label{eq:14}
(1 - \d) \nmeusq{u} \leq \nmeusq{Au} \leq (1 + \d) \nmeusq{u}  , \fa u\in\SI_k.
\ee
\end{definition}

The RIP is formulated in \cite{Candes-Tao05}.
It is shown in a series of papers \cite{Candes-Tao05,CRT06b,Candes08}
that the RIP of $A$ is sufficient for $(A,\DBP)$ to achieve
robust sparse recovery.
The best known and indeed the ``best possible,'' result
relating RIP and robust recovery is given below:

\begin{theorem}\label{thm:CZ}
If $A$ satisfies the RIP of order $tk$ with constant $\d_{tk} 
< \sqrt{(t-1)/t}$ for $t \geq 4/3$,
or $\d_{tk} < t/(4-t)$ for $t \in (0,4/3)$,
then $(A,\DBP)$ achieves robust sparse recovery of order $k$.
Moreover, both bounds are tight.
\end{theorem}

The first bound is proved in \cite{CZ14} while the second bound is proved
in \cite{Zhang-Li-TIT18}.
Note that both bounds are equal when $t = 4/3$.
Hence the theorem provides a continuous tight bound on
$\d_{tk}$ for all $t > 0$.

This theorem raises the question as to how one may go about designing
measurement matrices that satisfy the RIP.
There are two popular approaches, one probabilistic and one deterministic.
In the probabilistic method, the measurement matrix $A$ equals $(1/\sqrt{m})
\Phi$, where $\Phi$ consists of $mn$ independent samples of a 
Gaussian random variable, or more generally, a sub-Gaussian random variable.
In this paper we restrict our attention to the case, where $A$ consists of
random samples from a Gaussian distribution and refer the reader to \cite{FR13} for the more
general case of sub-Gaussian samples.
The relevant bound on $m$ to ensure that $A$ satisfies the RIP
with high probability is given next; it is a fairly straight-forward
modification of \cite[Theorem 9.27]{FR13}.

\begin{theorem}\label{thm:11}
Suppose an integer $k$ and real numbers $\d , \xi \in (0,1)$ are specified
and that $A = (1/\sqrt{m}) \Phi$, where $\Phi \in \R^{m \times n}$
consists of independent samples of a normal Gaussian random variable $X$.
Define
\be\label{eq:02}
g = 1 + \frac{1}{ \sqrt{2 \ln( en/k) } } ,
\eta = \frac{ \sqrt{1 + \d } - 1 }{g} .
\ee
Then $A$ satisfies the RIP of order $k$ with constant $\d$ with probability
at least $1 - \xi$ provided
\be\label{eq:03}
m \geq \frac{2}{\eta^2} \left( k \ln \frac{en}{k} + \ln \frac{2}{\xi}
\right) .
\ee
\end{theorem}

\begin{proof}
% The proof of this theorem is given in very sketchy form,
% as it follows that of 
We start with \cite[Theorem 9.27]{FR13}.
In that theorem, it is shown that, if the measurement matrix
$A \in \Rmn$ consists of independent samples of Gaussian
random variables and if
\bd
m \geq \frac{2}{\eta^2} \left( k \ln \frac{en}{k} + \ln \frac{2}{\xi}
\right) ,
\ed
where $\eta$ satisfies
\bd
\d \leq 2g \eta + g^2 \eta^2 ,
\ed
then $A$ satisfies the RIP of order $k$ with constant $\d$, with
probability at least $1 - \xi$.
The above equation can be rewritten as
\bd
\d + 1 \leq 1 + 2 g \eta + g^2 \eta^2 = (1 + g \eta)^2 .
\ed
Rearranging this equation leads to \eqref{eq:02}.
\end{proof}

Equation \eqref{eq:03} leads to an upper bound of the form 
$m = O(k \ln(n/k))$ for the number of measurements that suffice for
the random matrix to satisfy the RIP with high probability.
It is shown in \cite[Theorem 3.1]{BIPW10}
that \textit{any} algorithm that achieves stable sparse recovery requires
$m = O(k \ln(n/k))$ measurements.
See \cite[Theorem 5.1]{Cohen-Dahmen-DeVore09} for an earlier version.
For the convenience of the reader, we restate the theorem from
\cite{BIPW10}.
Note that it is assumed in \cite{BIPW10} that $p = q = 1$, but the
proof requires only that $p = q$.
In order to state the theorem, we introduce the entropy with respect to
an arbitrary integer $\th$.
Suppose $\th \geq 2$ is an integer.
Then the \textbf{$\th$-ary entropy} $H_\th : (0,1) \ap (0,1]$
is defined by
\be\label{eq:04}
H_\th(u) := -u \log_\th \frac{u}{\th-1} - (1-u) \log_\th (1-u) .
\ee

\begin{theorem}\label{thm:12}
Suppose $A \in \Rmn$ and that, for some map $\D: \R^m \ap \R^n$,
the pair $(A,\D)$ achieves stable $k$-sparse recovery with constant $C$.
Define $\th = \lfloor n/k \rfloor$.
Then
% \footnote{Note that the base of the logarithm does not matter because
% it cancels out between the two $\log$ terms.}
\be\label{eq:05}
m \geq \frac{1 - H_\th(1/2)}{\ln(4+2C)} k \ln \th
\ee

\end{theorem}

Because robust $k$-sparse recovery implies stable $k$-sparse recovery,
the bound in \eqref{eq:05} applies also to robust $k$-sparse recovery.

Comparing Theorems \ref{thm:11} and \ref{thm:12} shows that
$m = O(k \ln(n/k))$ measurements are both necessary and sufficient
for robust $k$-sparse recovery.
For this reason, the probabilistically generated measurement matrices
are considered to be ``order-optimal.''
However, this statement is misleading because the $O$ symbol in the
upper bound hides a very large constant, as shown next.

\begin{example}\label{exam:21}
Suppose $n = 22,201 = 149^2$ and $k = 69$, which is a problem instance studied
later in Section \ref{sec:num}.
Then the upper and lower bounds from Theorems \ref{thm:11}
and \ref{thm:12} imply that
\bd
14 \leq m \leq 44,345.
\ed
Thus the spread between the upper and lower bounds is more than
three orders of magnitude.
Also, the upper bound for the number of measurements is \textit{more} than the dimension $n$.
\end{example}

There is another factor as well.
As can be seen from Theorem \ref{thm:11},
probabilistic methods lead to measurement
matrices that satisfy the RIP \textit{only with high probability},
that can be made close to one but never exactly equal to one.
Moreover, as shown in \cite{BDMS13}, once a matrix has been generated,
it is NP-hard to test whether \textit{that particular matrix} satisfies the RIP.

These observations have led the research community to explore deterministic
methods to construct matrices that satisfy the RIP.
A popular approach is based on the coherence of a matrix.

\begin{definition}\label{def:coh}
Suppose $A \in \Rmn$ is column-normalized, so that $\nmeu{a_j} = 1$
for all $j \in [n]$, where $a_j$ denotes the $j$-column of $A$.
Then the \textbf{coherence} of $A$ is denoted by $\mu(A)$ and is defined as
\be\label{eq:15f}
\mu(A) := \max_{i \neq j} | \IP{a_i}{a_j} | .
\ee
\end{definition}

The following result is an easy consequence of the Gerschgorin circle theorem.

\begin{lemma}\label{lemma:11}
A matrix $A \in \Rmn$ satisfies the RIP of order $k$ with constant
\be\label{eq:15g}
\d_k = (k-1) \mu ,
\ee
provided that $(k-1) \mu < 1$, or equivalently, $k < 1 + 1/\mu$.
\end{lemma}

\subsection{Approaches to Compressed Sensing -- II: RNSP}\label{ssec:RNSP}

An alternative to the RIP approach to compressed sensing is provided by
the stable (and robust) null space property.
The SNSP is formulated in \cite{Xu-Hassibi08b}, while, to the best of the
authors' knowledge, the RNSP is formulated for the first time in
\cite{Foucart14}; see also \cite[Definition 4.17]{FR13}.

\begin{definition}\label{def:NSP}
Suppose $A \in \Rmn$ and let $\N(A)$ denote the null space of $A$.
Then $A$ is said to satisfy the \textbf{stable null space property (SNSP)}
of order $k$ with constant $\r < 1$ if, for every set $S \seq [n]$
with $|S| \leq k$, we have that
\be\label{eq:15}
\nmm{v_S}_1 \leq \r \nmm{v_{S^c}}_1 , \fa v \in \N(A) .
\ee
The matrix $A$ is said to satisfy the \textbf{robust null space property (RNSP)}
of order $k$ for the norm $\nmm{\cdot}$
with constants $\r < 1$ and $\t > 0$ if, 
for every set $S \seq [n]$ with $|S| \leq k$, we have that
\be\label{eq:16}
\nmm{h_S}_1 \leq \r \nmm{h_{S^c}}_1 + \t \nmm{Ah} , \fa h \in \R^n .
\ee
\end{definition}

It is obvious that RNSP implies the SNSP.
The utility of these definitions is brought out in the following theorems.

\begin{theorem}\label{thm:SNSP}
(See \cite[Theorem 4.12]{FR13}.)
Suppose $A$ satisfies the stable null space property of order $k$
with constant $\r$.
Then the pair $(A,\DBP)$ achieves stable $k$-sparse recovery with
\be\label{eq:17}
C = 2 \frac{ 1+\r }{ 1 - \r } .
\ee
\end{theorem}

\begin{theorem}\label{thm:RNSP}
(See \cite[Theorem 4.22]{FR13}.)
Suppose $A$ satisfies the robust null space property of order $k$
for the norm $\nmm{\cdot}$ with constants $\r$ and $\t$.
Then the pair $(A,\DBP)$ achieves robust $k$-sparse recovery with
\be\label{eq:18}
C = 2 \frac{ 1+\r }{ 1 - \r } , D = \frac{ 4 \t}{1-\r} .
\ee
\end{theorem}

\subsection{Best Bounds on the Sparsity Count Using the RIP}

Until recently, the twin approaches of RIP and RNSP had proceeded along
parallel tracks.
However, it is shown in \cite[Theorem 9]{MV-Ranjan19} that
if $A$ satisfies the RIP of order $tk$ with constant $\d_{tk} 
< \sqrt{(t-1)/t}$ for some $t > 1$, then it satisfies the RSNP of order $k$.
Note that if $A$ has coherence $\mu$, then by Lemma \ref{lemma:11},
we have that $\d_{tk} \leq (tk-1) \mu$ for all $t$.
Next by \cite[Theorem 9]{MV-Ranjan19},  
basis pursuit achieves robust $k$-sparse recovery whenever
\be\label{eq:111}
(tk-1) \mu < \sqrt{ \frac{t-1}{t} } 
\ee
for \textit{any} $t > 1$.
So let us ask: What is an ``optimal'' choice of $t$?
To answer this question, we neglect the $1$ in comparison to $tk$
and rewrite the above inequality as
\bd
k \mu < \sqrt{ \frac{t-1}{t^3} } .
\ed
Thus we get the best bound by maximizing the right side with respect to $t$.
It is an easy exercise in calculus to show that the maximum is achieved
with $t = 3/2$ and the corresponding bound $\sqrt{(t-1)/t} = 1/\sqrt{3}$.
Hence by combining with Lemma \ref{lemma:11} we can derive the following bound.

\begin{theorem}\label{thm:opt}
Suppose $A \in \Rmn$ has coherence $\mu$.
Then $(A,\DBP)$ achieves robust $k$-sparse recovery whenever
\be\label{eq:112}
((3/2) k -1 ) \mu < 1/\sqrt{3} ,
\ee
or equivalently
\be\label{eq:113}
k < \left\lfloor \frac{2}{3 \sqrt{3} \mu} + \frac{2}{3} \right\rfloor .
\ee
Moreover, the bound is nearly optimal when applying \cite[Theorem 9]{MV-Ranjan19} .

\end{theorem}

If we retain the term $tk-1$ instead of replacing it by $tk$, we would
get a more complicated expression for the optimal value of $t$.
However, it can be verified that if \eqref{eq:112} is satisfied, then 
so is \eqref{eq:111}.

\subsection{Binary Matrices for Compressed Sensing: A Review}\label{ssec:binary}

In this section we present a brief review of the use of binary matrices
as measurement matrices in compressed sensing.
The first construction of a binary matrix that satisfies the RIP
is due to DeVore and is given in \cite{DeVore07}.
The DeVore matrix has dimensions $q^2 \times q^{r+1}$,
where $q$ is a power of a prime number and $r \geq 2$ is an integer,
has exactly
$q$ elements of $1$ in each column and has coherence $\mu \leq r/q$.
This construction is generalized to algebraic curves in
\cite{Li-Gao-Ge-Zhang-TIT12}, but does not seem to offer much of an
advantage over that in \cite{DeVore07}.
A construction that leads to matrices of order $2^m \times 2^{m(m+1)/2}$
based on Reed-Muller codes is proposed in \cite{Howard-et-al-ICSS08}.
Because the number of measurements is restricted to be a power of $2$,
this is not a very practical method.
A construction in \cite{Naidu-et-al-TSP16} is based on a method to
generate Euler squares
from nearly a century ago \cite{MacNeish-AM22}.
The resulting binary matrix has dimensions $lq \times q^2$, where
$q$ is an \textit{arbitrary} integer, making this perhaps the most
versatile construction.
The integer $l$ is bounded as follows:
Let $q = 2^{r_0} p_1^{r_1} \ldots p_s^{r_s}$ be the prime number
decomposition of $q$.
Then $l+1 \leq \min \{ 2^{r_0}, p_1^{r_1}, \ldots , p_s^{r_s} \}$.
In particular if $q$ is itself a power of a prime, we can have $l = q-1$.
Each column of the resulting binary matrix has exactly $l$ ones
and the matrix has coherence $1/l$.
All of these matrices can be used to achieve robust $k$-sparse recovery
via the basis pursuit formulation, by combining Lemma \ref{lemma:11} with
Theorem \ref{thm:CZ}.
Another method found in \cite{Ehrich-et-al-TIT10} constructs binary
matrices using the Chinese remainder theorem and achieves
\textit{probabilistic} recovery.

There is another property that is sometimes referred to as the
$\ell_1$-RIP, introduced in
\cite{Indyk-Ruzic08},
which makes a connection between expander graphs and compressed sensing.
However, while this approach readily leads to stable $k$-sparse recovery,
it does not lend itself readily to \textit{robust} $k$-sparse recovery.
One of the main contributions of \cite{Mahsa-TSP18} is to show
that the construction of \cite{DeVore07} can also be viewed as a special
case of an expander graph construction proposed in \cite{Guruswami-et-al09}.

Yet another direction is initiated in \cite{Dimakis-et-al-TIT12},
in which a general approach is presented for generating binary matrices
for compressed sensing using algebraic coding theory.
In particular, it is shown that binary matrices which, when viewed as
elements over the binary field $\Fbb_2$, have good properties in decoding,
will also be good measurement matrices when viewed as matrices of real numbers.
In particular, several notions of ``pseudo-weights'' are introduced
and it is shown that these pseudo-weights can be related to the satisfaction
of the stable (but not robust) null space property of binary matrices.
These bounds are improved in \cite{Liu-Xia-ISIT13} to prove the
stable null space property under weaker conditions than in
\cite{Dimakis-et-al-TIT12}.

\section{Robust Null Space Property of Binary Matrices}\label{sec:RNSP}

In this section we commence presenting the new results of this paper on
identifying a class of binary matrices for compressed sensing that have
a nearly optimal number of measurements.

Suppose $A \in \bi^{m \times n}$ with $m < n$.
Then $A$ can be viewed as the bi-adjacency matrix of a bipartite graph with
$n$ input (or ``left'') nodes and $m$ output (or ``right'') nodes.
Such a graph is said to be \textbf{left-regular} if each input node has the
same degree, say $d_L$.
This is equivalent to saying that each column of $A$ contains
exactly $d_L$ ones.
Given a bipartite graph with $E$ edges, $n$ input
nodes and $m$ output nodes, define the ``average left degree'' 
and ``average right degree'' of the graph
as $\db_L = E/n$ and $\db_R = E/m$.
Note that these average degrees need not be integers.
Then it is clear that $n \db_L = m \db_R$.
The \textbf{girth} of a graph is defined as the length of the shortest cycle.
Note that the girth of a bipartite graph is always an even number and
in so-called simple graphs (not more than one edge between any pair of vertices),
the girth is at least four.

Hereafter, we will not make a distinction between a binary
matrix and the bipartite graph associated with the matrix.
Specifically, the columns correspond to the ``left'' nodes while the
rows correspond to the ``right'' nodes.
So an expression such as ``$A$ is a left-regular binary matrix of degree $d_L$''
means that the associated bipartite graph is left-regular with degree $d_L$.
This usage will permit us to avoid some tortuous sentences.

Theorems \ref{thm:LX1} and \ref{thm:LX2} are
the starting point for the contents of this section.

\begin{theorem}\label{thm:LX1}
(See \cite[Theorem 2]{Liu-Xia-ISIT13}.)
Suppose $A \in \bimn$ is left-regular 
with left degree $d_L$ and suppose that the maximum inner product between
any two columns of $A$ is $\l$.
Then for every $v \in \N(A)$, we have that
\be\label{eq:32}
| v_i | \leq \frac{ \l }{ 2 d_L } \nmm{v}_1 , \fa i \in [n] ,
\ee
where $[n]$ denotes $\{ 1 , \ldots , n \}$.
\end{theorem}

If the matrix $A$ has girth six or more, then the maximum inner product
between any two columns of $A$ is at most equal to one.
Therefore \eqref{eq:32} gives the bound
\bd
| v_i | \leq \frac{ 1 }{ 2 d_L } \nmm{v}_1 , \fa i \in [n] .
\ed
However, if the girth is equal to $10$ or more, then it is possible
to improve the bound \eqref{eq:32}.

\begin{theorem}\label{thm:LX2}
(See \cite[Theorem 3]{Liu-Xia-ISIT13}.)
Suppose $A \in \bimn$ and that $A$ has girth $g \geq 6$.
Then for every $v \in \N(A)$, we have that
\be\label{eq:35}
| v_i | \leq \frac{ \nmm{v}_1 }{C'} , \fa i \in [n] ,
\ee
where, if $g = 4t+2$, then
\be\label{eq:36}
C' := 2 \sum_{i=0}^t (d_L-1)^i ,
\ee
and if $g = 4t$, then
\be\label{eq:37}
C' := 2 \sum_{i=0}^{t-1} (d_L-1)^i ,
\ee
\end{theorem}

Note that if the girth of the graph
equals $6$, then $C'$ as defined in \eqref{eq:36} becomes
$C' = 2$ and the bound in \eqref{eq:35} becomes the same as that in
\eqref{eq:32} after noting that $\l = 1$.
Similarly, if $g = 8$, then $\C'$ in \eqref{eq:37} also becomes just $C' = 2$.
Therefore Theorem \ref{thm:LX2} is an improvement over Theorem \ref{thm:LX1}
only when the girth of the graph is at least equal to $10$.

In \cite{Liu-Xia-ISIT13}, the bounds \eqref{eq:32} and \eqref{eq:35}
are used to derive
sufficient conditions for the matrix $A$ to satisfy the \textit{stable}
null space property.
However, it is now shown that the same two bounds can be used to infer
the \textit{robust} null space property of $A$.
This is a substantial improvement, because with such an $A$ matrix,
basis pursuit would lead to \textit{robustness against measurement noise},
which is not guaranteed with the SNSP.
We derive our results through a series of preliminary results.

\begin{lemma}\label{lemma:1}
Suppose $A \in \Rmn$ and let $\nmm{\cdot}$ be any norm on $\R^m$.
Suppose there exist constants $\al > 2, \beta > 0$ such that
\be\label{eq:38a}
| h_i | \leq \frac{ \nmm{h}_1 }{ \al } + \beta \nmm{Ah} ,
\fa i \in [n] , \fa h \in \R^n .
\ee
Then, for all $k < \al/2$, the matrix $A$ satisfies the RNSP of order $k$.
Specifically, whenever $S \seq [n]$ with $|S| \leq k$, Equation
\eqref{eq:16} holds with
\be\label{eq:38b}
\r = \frac{k}{\al - k} , \t = \frac{ \al k \beta}{\al - k } .
\ee
\end{lemma}

\begin{proof}
Let $S \seq [n]$ with $|S| \leq k$ be arbitrary.
Then
\begin{eqnarray*}
\nmm{h_S}_1 & = & \sum_{i \in S} |h_i| \\
& \leq & \frac{k}{\al} \nmm{h}_1 + k \beta \nmm{Ah} \\
& = & \frac{k}{\al} ( \nmm{ h_S }_1 + \nmm{ h_{S^c} }_1 ) + k \beta \nmm{Ah} .
\end{eqnarray*}
Therefore
\bd
\left( 1 - \frac{k}{\al} \right) \nmm{ h_S }_1 
\leq \frac{k}{\al} \nmm{ h_{S^c} }_1 + k \beta \nmm{Ah} ,
\ed
or
\bd
\nmm{ h_S }_1 \leq \frac{k}{\al - k} \nmm{ h_{S^c} }_1
+ \frac{ \al k \beta }{ \al - k } \nmm{Ah} ,
\ed
which is the desired conclusion.
\end{proof}

Next, let $A \in \Rmn$ be arbitrary and let $\nmm{\cdot}$ be any norm
on $\R^n$.
Recall that $\N(A) \seq \R^n$ denote the null space of $A$ and let
$\N^\perp := [ \N(A) ]^\perp$ denote the orthogonal complement of $\N(A)$
in $\R^n$.
Then for all $u \in \N^\perp$, it is easy to see that
\bd
\nmeu{u} \leq \frac{1}{\s_{{\rm min}} } \nmeu{Au} ,
\ed
where $\s_{{\rm min}} $ is the smallest nonzero singular value of $A$.
Because all norms on a finite-dimensional space are equivalent,
there exists a constant $c$ that depends only on the norm $\nmm{\cdot}$
on $\R^m$ such that
\be\label{eq:38c}
\nmeu{y} \leq c \nmm{y} , \fa y \in \R^m .
\ee
(In particular, $\nmeu{y} \leq \nmm{y}_1$, so we can take $c = 1$
in this case.)
Therefore, by Schwarz' inequality, we get
\be\label{eq:38d}
\nmm{u}_1 \leq \sqrt{n} \nmeu{u} \leq \frac{c \sqrt{n} }{\s_{{\rm min}} }
\nmm{Au} , \fa u \in \N^\perp .
\ee

At this point, we can state the main result of this section.

\begin{theorem}\label{thm:1}
Suppose $A \in \bimn$ is left-regular 
with left degree $d_L$ and let $\l$ denote the maximum inner product between
any two columns of $A$ (and observe that $\l \leq d_L$).
Next, let $\s_{{\rm min}}$ denote the smallest nonzero singular value of $A$ 
and for an arbitrary norm $\nmm{\cdot}$ on $\R^m$, choose the constant
$c$ such that \eqref{eq:38c} holds.
Then $A$ satisfies \eqref{eq:38a} with
\be\label{eq:38h}
\al = \frac{2 d_L}{\l} ,
\beta = \left( \frac{\l}{2 d_L} + 1 \right) \frac{c \sqrt{n} }{\s_{{\rm min}} } .
\ee
Consequently, for all $k < \al/2 = d_L/\l$,
$A$ satisfies the RNSP of order $k$ with
\be\label{eq:38e}
\r = \frac{ \l k }{ 2 d_L - \l k} ,
\t = \frac{2 d_L k}{2 d_L - \l k } \beta .
\ee
\end{theorem}

\begin{proof}
Let $h \in \R^n$ be arbitrary and express $h$ as $h = v + u$, where
$v \in \N(A)$ and $u \in \N^\perp$.
Then clearly
\bd
|h_i| = | v_i + u_i | \leq |v_i| + |u_i| , \fa i \in [n] .
\ed
We will bound each term separately.

As shown in Theorem \ref{thm:LX1}, we have that
\begin{eqnarray*}
|v_i| & \leq & \frac{\l}{2 d_L} \nmm{v}_1 \\
& \leq & \frac{\l}{2 d_L} ( \nmm{h}_1 + \nmm{u}_1 ) \\
& \leq & \frac{\l}{2 d_L} \nmm{h}_1 + 
\frac{\l c \sqrt{n} }{2 d_L \s_{{\rm min}} } \nmm{Au} \\
& = & \frac{\l}{2 d_L} \nmm{h}_1 + 
\frac{\l c \sqrt{n} }{2 d_L \s_{{\rm min}} } \nmm{Ah} ,
\end{eqnarray*}
where the last step follows from the fact that $Ah = Au$ because $Av = \bz$.
Next
\bd
|u_i| \leq \nmm{u}_1 \leq \frac{c \sqrt{n} }{\s_{{\rm min}} } \nmm{Ah} ,
\fa i \in [n] .
\ed
Combining these two inequalities shows that
\bd
|h_i| \leq |v_i| + |u_i| \leq \frac{\l}{2 d_L} \nmm{h}_1 +
\left( \frac{\l}{2 d_L} + 1 \right) \frac{c \sqrt{n}}{\s_{{\rm min}} } \nmm{Ah} .
\ed
This establishes \eqref{eq:38h}.
The \eqref{eq:38e} follows from Lemma \ref{lemma:1}, specifically \eqref{eq:38b}.
\end{proof}

\textbf{Remarks:}
\bit
\item In the above proof, we make use of the inequality
$|u_i| \leq \nmm{u}_1$.
At a certain level, this estimate is conservative.
However, if we wish to have a bound on $|u_i|$ in terms of $\nmm{u}_1$
that is applicable to \textit{all} vectors $u$, then the bound is tight.
\item Note that the bound $|u_i| \leq \nmm{u}_1$ is used \textit{only} 
to derive a bound on the constant $\beta$.
In turn the bound on $\beta$ leads to a bound on the constant $\t$ in
the definition of the robust null space property.
It can be seen from Theorem \ref{thm:RNSP} and \eqref{eq:18} that
robust $k$-sparse recovery occurs whenever $\r < 1$ and the only appearance
of $\t$ is in the constant $D$ in \eqref{eq:18}, which gives the
amplification factor of the noise.
\eit

\begin{theorem}\label{thm:2}
Suppose $A \in \bimn$ is left-regular with left-degree $d_L$ and has girth
at least six.
Define the constant $C'$ as in \eqref{eq:36} or \eqref{eq:37}
as appropriate.
Then for all $k < C'/2$, the matrix $A$ satisfies the RNSP of order $k$,
with constants
\be\label{eq:38f}
\r = \frac{k}{C' - k} , \t = \frac{C' - k}{C' k} \beta .
\ee
\end{theorem}

The proof of Theorem \ref{thm:2} is entirely analogous to that of
Theorem \ref{thm:1}, with the bound in Theorem \ref{thm:LX2} replacing
that in Theorem \ref{thm:LX1}.
Therefore the proof is omitted.

The results in Theorem \ref{thm:1} lead to sharper bounds for the
sparsity count compared to using RIP and coherence bounds.
This is illustrated next.

\begin{example}\label{exam:31}
Suppose $A \in \bimn$ is left-regular with degree $d_L$ and with the
inner product between any two columns bounded above by $\l$.
Then it is easy to see that the coherence $\mu$ of $A$ is bounded by $\l/d_L$.
Therefore, if we use Theorem \ref{thm:opt}, then it follows that $(A,\DBP)$ 
achieves robust $k$-sparse recovery whenever
\bd
k < \left\lfloor \frac{ 2 d_L}{3 \sqrt{3} \l } + \frac{2}{3} \right\rfloor .
\ed
In contrast, if we use Theorem \ref{thm:1}, it follows that
$(A,\DBP)$ achieves robust sparse recovery whenever $k < d_L/\l$,
which is an improvement by a factor of roughly $3 \sqrt{3}/2 \approx 2.6$.
\end{example}

\section{Lower Bounds on the Number of Measurements}\label{sec:Lower}

Theorem \ref{thm:LX2} shows that, for a fixed left degree $d_L$,
as the girth of the graph corresponding to
$A$ becomes larger, so does the constant $C'$.
Therefore, as the girth of $A$ increases, so does the upper bound on $k$
as obtained from Theorem \ref{thm:2}.
This suggests that, for a given left degree $d_L$ and
number of input nodes $n$, it is better to choose graphs of large girth.
However, as shown next,
as the girth of a graph is increased, the number of measurements $m$
also increases.
As shown below, the ``optimal'' choice for the girth is actually $6$.

Observe from Theorem \ref{thm:2} and specifically
\eqref{eq:38f}, that the pair $(A,\DBP)$ achieves robust $k$-sparse
recovery whenever $\r < 1$, or equivalently $k < C'/2$.
From the definition of $C'$, this bound on the sparsity count
for which robust $k$-sparse recovery is guaranteed can be written as
\be\label{eq:36a}
k < \sum_{i=0}^t (d_L-1)^i ,
\ee
if $g = 4t+2$ and
\be\label{eq:37a}
k < \sum_{i=0}^{t-1} (d_L-1)^i ,
\ee
and if $g = 4t$.
Let us define
\be\label{eq:38}
\kb := \left\{ \ba{ll}
(d_L - 1)^t & \mbox{if } g = 4t + 2 , \\
(d_L - 1)^{t-1} & \mbox{if } g = 4t .
\ea \right.
\ee
It is recognized that $\kb$ is just the last term in the summations
in \eqref{eq:36a} and \eqref{eq:37a}.
Moreover, unless $d_L$ is quite small, the difference between $\kb$
and the summations in \eqref{eq:36a} and \eqref{eq:37a} will be rather small.
Thus we use $k < \kb$ as an easily analyzable
and quite reasonable, approximation
to the actual upper bounds on the sparsity count $k$ given
in \eqref{eq:36a} and \eqref{eq:37a}.

% Therefore, if the actual sparsity count $k$ satisfies $k \leq \kb$, then
% it follows from Theorem \ref{thm:2} that the pair $(A,\DBP)$ achieves
% robust $k$-sparse recovery.
It is clear that if we choose the matrix $A$ to have higher and higher
girth, the bound $\kb$ also becomes higher.
So the question therefore becomes: What happens to $m$, the number
of measurements, as the girth is increased?
The answer is given next.

\begin{theorem}\label{thm:3}
Suppose $A \in \bimn$ is $d_L$-left
regular graph with $m \leq n$ and that every row and every column of $A$
contains at least two ones.
If the girth $g$ of $A$ equals $4t+2$, then
\be\label{eq:39}
m \geq \kb^{2/(t+1)} n^{t/(t+1)} ,
\ee
whereas if $g = 4t$ for $t \geq 2$, then
\be\label{eq:310}
m \geq \kb^{(2t-1)/[t(t-1)]} n^{(t-1)/t} .
\ee
\end{theorem}

The proof of Theorem \ref{thm:3} is based on
the following result \cite[Equations (1) and (2)] {Hoory-JCTB02}:

\begin{theorem}\label{thm:Hoory}
Suppose $A \in \bi^{m \times n}$ with $m < n$.
Suppose further that in the bipartite graph associated with $A$, every 
node has degree $\geq 2$.\footnote{This is equivalent to the requirement
that every row and every column of $A$ contains at least two ones.}
Let $E$ denote the total number of edges of the graph and define
$\db_L = E/n, \db_R = E/m$ to be the average left-node degree and average
right-node degree, respectively.
Suppose finally that the graph has girth $g = 2r$.
Then
\be\label{eq:311}
m \geq \sum_{i=0}^{r-1} ( \db_L - 1 )^{\lceil i/2 \rceil}
( \db_R - 1 )^{\lfloor i/2 \rfloor} .
\ee
\end{theorem}

It is important to note that the above theorem does not require any
assumptions about the underlying graph (e.g., regularity).
The only assumption is that every node has degree two or more, so as to
rule out trivial cases.
Usually such theorems are used to find upper bounds on the girth of a
bipartite graph in terms of the numbers of its nodes and edges
(as in Theorem \ref{thm:4} below).
However, we turn it around here and use the theorem to find a lower
bound on $m$, given the integers $n$ and $g$.
Note that if $g = 4$, then $r = 2$ and the bound \eqref{eq:311} becomes
$m \geq \db_L$, which is trivial.
In fact $m$ has to exceed the \textit{maximum} degree of any left node.
However, for $g \geq 6$, the bound in \eqref{eq:311} is meaningful.

\begin{proof}
(Of Theorem \ref{thm:3}:)
The bound \eqref{eq:311} implies that $m$ is no smaller than the
last term in the summation; that is
\be\label{eq:312}
m \geq \db_L^{\lceil (r-1)/2 \rceil} \db_R^{\lfloor (r-1)/2 \rfloor} .
\ee
Because $A$ is assumed to be left-regular, actually $\db_L = d_L$, but
we do not make use of this and will carry the symbol $\db_L$ throughout.
By definition, we have that $\db_R = (n \db_L)/m$.
Therefore, if $n \geq m$, then it follows that
\bd
\db_R - 1 = \frac{n \db_L}{m} - 1 \geq \frac{n \db_L}{m} - \frac{n}{m}
= \frac{n}{m} ( \db_L - 1) .
\ed
Therefore \eqref{eq:312} implies that
\be\label{eq:313}
m \geq ( \db_L - 1 )^\al \left( \frac{n}{m} \right)^{\lfloor (r-1)/2 \rfloor} ,
\ee
where
\bd
\al = \lceil (r-1)/2 \rceil + \lfloor (r-1)/2 \rfloor .
\ed

We treat the cases $g = 4t+2$ and $g = 4t$ separately.
If $g = 4t+2$, then $r = g/2 = 2t+1$, so that
\bd
\lceil (r-1)/2 \rceil = \lfloor (r-1)/2 \rfloor = t , \al = 2t .
\ed
Therefore \eqref{eq:313} becomes
\bd
	m \geq ( \db_L - 1 )^{2t} \left( \frac{n}{m} \right)^t
= \kb^2 \left( \frac{n}{m} \right)^t .
\ed
This can be rearranged as
\bd
m^{t+1} \geq n^t \kb^2 , 
\ed
or
\bd
m \geq \kb^{2/(t+1)} n^{t/(t+1)} ,
\ed
which is \eqref{eq:39}.
In case $g = 4t$, the proof proceeds along entirely parallel lines
and is omitted.
\end{proof}

It is obvious from \eqref{eq:39} that the lower bound is minimized
(for a fixed choice of $n$ and $\kb$) with $t = 1$, or $g = 6$.
Similarly, the lower bound in \eqref{eq:310} is minimized when $t = 2$,
or $g = 8$.
Higher values of $g$ would lead to more measurements being required.
We can also compare $g = 6$ with $g = 8$ and show that $g = 6$ is better.
Let us substitute $t = 1$ in \eqref{eq:39} and $t = 2$ in \eqref{eq:310}.
This gives
\be\label{eq:314}
m \geq \left\{ \ba{ll} \kb n^{1/2} & \mbox{if } g = 6 , \\
\kb^{3/2} n^{1/2} & \mbox{if } g = 8 . \ea \right.
\ee
If we wish to have fewer measurements than the dimension of the unknown
vector, we can set $m < n$.
Substituting this requirement into \eqref{eq:314} leads to
\bd
\kb < n^{1/2} \mbox{ if } g = 6 , \kb < n^{1/3} \mbox{ if } g = 8 .
\ed
Hence graphs of girth 6 are preferable to graphs of girth 8, because the
upper limit on the recoverable sparsity count $\kb$ is higher with a graph
of girth $6$ than with a graph of girth $8$.

\section{Construction of Nearly Optimal Graphs of Girth Six}\label{sec:Girth}

The discussion of the preceding section suggests that we must look for
bipartite graphs of girth six, where the integer $m$ satisfies the bound
\eqref{eq:311} with the $\geq$ replaced by an equality, or at least, 
close to it.
In this section
it is shown
% we prove a general result to the effect
that a certain class of binary matrices has girth six.
Then we give two specific constructions.
The first of these is based on array codes, which are a class of
low density parity check (LDPC) codes and the second is based
on Euler squares.The first construction is easier to explain, but the second one gives
far more flexibility in terms of the number of measurements.
Here is the general theorem.

\begin{theorem}\label{thm:4}
Suppose $A \in \bi^{lq \times q^2}$ for some integers $4 \leq l \leq q-1$.
Suppose further that
\ben
\item $\db_L \geq l$, where $\db_L$ is the average left degree of $A$.
\item The maximum inner product between any two columns of $A$ is one.
\item Every row and every column of $A$ have at least two ones.
\een
Then the girth of $A$ is six.
\end{theorem}

\textbf{Remark:}
Before proving the theorem,
let us see how closely such a matrix satisfies the inequality \eqref{eq:311}.
In the constructions below we have that $\db_L = d_L = l$, $g = 6$ and $r = 3$.
Therefore the bound in \eqref{eq:311} becomes
\bd
m \geq 1 + (l-1) + (l-1)(q-1) = q(l-1) + 1 .
\ed
Since $m = lq$, we see that the actual value of $m$ exceeds the lower bound 
for $m$ by a factor of $l/(l-1)$ (after neglecting the last term of $1$
on the right side).
Note that there is no guarantee that the lower bound in Theorem
\ref{thm:1} is actually achievable.
So the class of matrices proposed above (if they could actually be
constructed), can be said to be ``near optimal.''
In applying this theorem, we would choose $q$ such that $n \leq q^2$
and choose any desired $l \leq q-1$.
With such a measurement matrix, basis pursuit will achieve robust $k$-sparse
recovery up to $k < \lceil \sqrt{n} \rceil_p$,
where $\lceil x \rceil$ denotes the smallest \textit{prime number}
larger than or equal to $x.$
% that is, $k < \sqrt{n}$, more or less.

\begin{proof}
Let $g$ denote the girth of $A$.
Then Condition (2) implies that $g \geq 6$.
Condition (3) implies that the bound \eqref{eq:311} applies with $m = lq$,
$n = q^2$, $n/m = q/l$.
Let $g = 2r$ and define
\bd
\al = \lceil (r-1)/2 \rceil + \lfloor (r-1)/2 \rfloor ,
\beta = \lfloor (r-1)/2 \rfloor .
\ed
Then the inequality \eqref{eq:311} implies that
\bd
lq \geq (\db_L - 1)^\al (q/l)^\beta \geq (l-1)^\al (q/l)^\beta .
\ed
This can be rewritten as
\be\label{eq:B1}
(l-1)^\al \frac{q^{\beta - 1} }{l^{\beta+1} } \leq 1 .
\ee
Note that $g \geq 6$, so that $r \geq 3$, due to Condition (2).
We study two cases separately.

\textbf{Case (1)}: $g = 4t$ for some $t \geq 2$.
In this case
\bd
(r-1)/2 = t - 1/2 , \lceil (r-1)/2 \rceil = t, \lfloor (r-1)/2 \rfloor = t-1 ,
\ed
\bd
\al = 2t-1 , \beta = t-1 .
\ed
Therefore \eqref{eq:B1} becomes
\be\label{eq:B2}
(l-1)^{2t-1} \frac{q^{t-2}}{l^t} \leq 1 ,
\ee
or
\bd
q^{t-2} (l-1)^{t-1} \leq \left( \frac{l}{l-1} \right)^t \leq 2^t ,
\ed
because $l/(l-1) \leq 2$ for $l \geq 2$.
Also
\bd
q^{t-2} (l-1)^{t-1} \geq q^{t-2} (l-1)^{t-2} = [q(l-1)]^{t-2} .
\ed
Combining these inequalities gives
\bd
[q(l-1)]^{t-2} \leq 2^t , 
\ed
or
\be\label{eq:B3}
\left[ \frac{q(l-1)}{2} \right]^{t-2} \leq 2^2 = 4 .
\ee
It is shown that \eqref{eq:B3} cannot hold if $t \geq 3$.
If $t \geq 3$, then
\bd
\frac{q(l-1)}{2} \leq \left[ \frac{q(l-1)}{2} \right]^{t-2} \leq 4 ,
\ed
or $q(l-1) \leq 8$.
However, $q \geq 5$ and $l-1 \geq 3$, so this inequality cannot hold.
At this point, let us consider the possibility that $g = 8$, i.e., that $t = 2$.
In this case \eqref{eq:B2} becomes
\bd
(l-1)^3 \frac{1}{l^2} \leq 1 , \mbox{ or } (l-1)^3 \leq l^2 .
\ed
This inequality can hold only for $l = 1, 2, 3$ and not if $l \geq 4$.
Hence $A$ cannot have girth $4t$ for any $t \geq 2$.

\textbf{Case (2)}: $g = 4t+2$ for some $t \geq 1$.
In this case \bd
\lceil (r-1)/2 \rceil = \lfloor (r-1)/2 \rfloor = t ,
% \ed
% \bd
\al = 2t, \beta = t .
\ed
So \eqref{eq:B1} becomes
\be\label{eq:B4}
(l-1)^{2t} \frac{q^{t-1}}{l^{t+1}} \leq 1 .
\ee
As before, this can be rewritten as
\bd
q^{t-1} (l-1)^{t-1} \leq \left( \frac{l}{l-1} \right)^{t+1} \leq 2^{t+1} ,
\ed
or
\be\label{eq:B5}
\left[ \frac{q(l-1)}{2} \right]^{t-1} \leq 2^2 = 4 .
\ee
This inequality can hold if $t = 1$ because the left side equals $1$.
However, if $t > 1$, then \eqref{eq:B5} implies that
\bd
\frac{q(l-1)}{2} \leq \left[ \frac{q(l-1)}{2} \right]^{t-1} \leq 4 ,
\ed
or $q(l-1) \leq 8$, which is impossible.
Hence \eqref{eq:B5} implies that $t = 1$, or that $g = 6$.
\end{proof}

In what follows, we present two explicit constructions of binary matrices that
satisfy the conditions of Theorem \ref{thm:4}.
The first construction is taken from the theory of low density
parity check (LDPC) codes and is a generalization of \cite{Yang-Hell-TIT03}.
This type of construction for Low Density 
Parity Check codes (LDPC) was first introduced in \cite{Fan00}.
Let $q$ be a prime number and let $P \in \bi^{q \times q}$ be any
cyclic permutation of $[q]$.
In \cite{Yang-Hell-TIT03} $P$ is taken as
the shift permutation matrix defined by $P_{i,i-1} = 1$ and the rest zeros,
where $i-1$ is interpreted modulo $q$.
Then $P^q = I$, the identity matrix.
Let $l < q$ be any integer and define the matrix $H(q,l) \in
\bi^{lq \times q^2}$ as the block-partitioned matrix $[ M_{ij} ] ,
i \in [l] , j \in [q]$, where
\be\label{eq:41}
M_{ij} = P^{(i-1)(j-1)} .
\ee
More elaborately, the matrix $H(q,l)$ is given by
\be\label{eq:42}
H(q,l) = \left[ \ba{ccccc}
I & I & I & \ldots & I \\
I & P & P^2 & \ldots & P^{q-1} \\
I & P^2 & P^4 & \ldots & P^{2(q-1)} \\
\vdots & \vdots & \ddots & \vdots & \vdots \\
I & P^{l-1} & P^{2(l-1)} & \ldots & P^{(l-1)(q-1)} 
\ea \right] .
\ee
The matrix $H(q,l)$ is bi-regular, with left (column) degree $l$ and
right (row) degree $q$.
It is rank-deficient, having rank $(q-1)l + 1$.
In principle we could drop the redundant rows, but that would destroy
the left-regularity of the matrix, thus rendering the theory in this
paper inapplicable.
(However, the resulting matrix would still be right-regular.)
Moreover, due to the cyclic nature of $P$, it follows that
the inner product between any two columns of $H(q,l)$ is at most equal to one.

It is shown in \cite[Proposition 1]{Yang-Hell-TIT03} that $H(q,l)$ has
girth six, but here that statement follows from Theorem \ref{thm:4}.

The second construction is based on Euler squares.
In \cite{MacNeish-AM22}, a general recipe is given for constructing
generalized Euler squares.
This is used in \cite{Naidu-et-al-TSP16}
to construct an associated binary matrix of order $lq \times q^2$,
where $q$ is any arbitrary integer (in contrast with the construction
of \cite{Yang-Hell-TIT03}, which requires $q$ to be a prime number),
such that the maximum inner
product between any two columns is at most equal to one.
Again, by Theorem \ref{thm:4}, such matrices have girth six and
are thus nearly optimal for compressed sensing.
The upper bound on $l$ is defined as follows:
Let $q = 2^{r_0} p_1^{r_1} \ldots p_s^{r_s}$ be the prime number decomposition of $q$.
Then $l < \min \{ 2^{r_0}, p_1^{r_1}, \ldots , p_s^{r_s} \}$.
In particular if $q$ is a prime or a power of a prime, then
we can have $l < q - 1$.
It is easy to verify that, if $q$ is a prime, then
the construction in \cite{Naidu-et-al-TSP16} is the same as 
the array code construction of \cite{Yang-Hell-TIT03} with permuted columns.
For the case, where $q$ is a prime power, the construction is more elaborate 
and is not pursued further here.

\begin{example}\label{exam:81}
In this example we compare the number of samples required when
using the DeVore construction of \cite{DeVore07}
and a matrix that satisfies the hypotheses
of Theorem \ref{thm:4}, such as the array code matrix or the Euler square
matrix.
The conclusions are that: (i) 
When $k < \sqrt{n}/4$, the Devore construction
requires fewer measurements than the array code,
whereas when $\sqrt{n}/4 < k < \sqrt{n}$,
the array code type of matrix requires fewer measurements.
(ii) When $k > \sqrt{n}/2$, the DeVore construction requires
more measurements than $n$, the dimension of the unknown vector,
whereas the array code construction has $m < n$ whenever $k < \sqrt{n}$.

To see this, recall that the DeVore construction produces a matrix
of dimensions $q^2 \times q^{r+1}$ with the maximum inner product
between columns equal to $r$ and each column contains $q$ ones.
So if we choose $r = 2$, then $\l$ in Theorem \ref{thm:2} equals $2$,
while $d_L = q$.
Consequently the DeVore matrix satisfies the RNSP of order $k$ whenever
$k < q/2$ and the number of measurements $m_D$ equals $q^2 = 4k^2$,
Thus $m_D < n$ requires that $4k^2 < n$, or $k < \sqrt{n}/2$.
In contrast, a matrix of the type discussed in Theorem \ref{thm:4} has
dimensions $lq \times q^2$, where $n = q^2$ and $l = k+1$.
For this class of matrices, we have $\l = 1$ and $d_L = q$.
This matrix satisfies the RNSP whenever $k = l-1 < q$
and the number of measurements equals $lq = (k+1)q$.
Now $4k^2 < kq$ if and only if $k < q/4 = \sqrt{n}/4$.
Also $m_A = (k+1)q < n = q^2$ whenever $k+1 < q = \sqrt{n}$.
Here, in the interests of simplicity, we ignore the fact that $q$ has
to be a prime number in both cases and various rounding up operations.
\end{example}

\section{Low Girth in Compressed Sensing vs.\ High Girth in Coding Theory}\label{sec:Low}

As shown in the previous section, in compressed sensing left-regular bipartite
graphs of girth six are preferable to graphs with higher girths.
It is easy to understand why graphs of girth four are undesirable.
For left-regular graphs of column degree $d_L$ and girth four,
recovery is guaranteed only for $k < (d_L-1)/2$, whereas for 
left-regular graphs of column degree $d_L$ and girth six, recovery is
guaranteed for $k < d_L$, or twice as large a bound.
However, it is counter-intuitive
that graphs of \textit{still higher} girth are also inferior
to graphs of girth six when it comes to compressed sensing, because in
LDPC coding, the higher the girth, the better the decoding performance.
% In this brief section, we attempt to explain this disparity.

In order to explain this disparity, we quote
% We begin by quoting
verbatim a comment by one of the reviewers, who said:
\begin{quote}
Although it is correct that in the area of LDPC codes large girth helps in
the limit $n \to \infty$, in practice people use mostly parity-check
matrices with girth six.
The reason for this is that most of the gain is
by going from girth four to girth six.
Going to larger girth is mostly not
worthwhile because of the loss of flexibility in designing parity-check
matrices for the typical values of $n$ of interest.

Intuitively, it is clear that for a given code length, given variable
node degree distribution and given check node distribution, the larger
the required girth, the fewer Tanner graphs there will be.
(Clearly, if the girth requirement is beyond some bound, there will be
no Tanner graph.)
Writing down the relevant constraints is particularly convenient for the
popular class of quasi-cyclic LDPC codes. See, for example
\cite{Fossorier-TIT04,Sma-Von-TIT12}.

Many papers have empirically observed that going from girth four to
girth six brings the most benefit, with limited payoff beyond that.
A mathematical approach to understand this can be found in
\cite[Section 8.3]{Von-Koett-arxiv05},
which is the extended version of \cite{Koett-Von03}.
\end{quote}

The fact is that, while both coding and compressed sensing use binary matrices,
there are some significant differences between them.
In coding, the number of bit-flipping errors $k$ (which is analogous the
sparsity count in compressed sensing) is a linear multiple of $n$, say
$k = \al n$ for some $\al \in (0,1)$.
In this case the universal lower bound from Theorem \ref{thm:12} becomes
$m = O(n \al \ln(1/\al))$ and the challenge is to design codes, where the
number $m$ of parity check bits grows linearly with $n$.
In contrast, in compressed sensing, the emphasis is on the case, where $k$
grows \textit{sub-linearly} with respect to $n$ and the objective is to ensure
that the number of measurements $m$ also grows more slowly than $n$, though
faster than $k$.
In this setting, the \textit{rate of the code} defined as $1 -m/n$ approaches
$1$ as $n$ grows.
For this setting, as shown here, the optimal girth of the bipartite graph
is six.

\section{Numerical Experiments}\label{sec:num}

In this section we carry out various numerical experiments to illustrate
the use of the array code binary matrices proposed in this paper.
The experiments include a comparison of the array code binary matrix and
the DeVore construction of binary matrices from \cite{DeVore07},
with random Gaussian matrcies.
In Section \ref{ssec:num1}, we compute the number of measurements
that are sufficient to \textit{guarantee} the recovery of $k$-sparse
$n$-dimensional vectors, for each of these classes of measurement matrices.
We also compute the CPU time for $\ell_1$-norm minimization to be
performed using each class of matrices.
While the absolute CPU time is not meaningful, the relative values are
indeed meaningful.
In Section \ref{ssec:num2} we study the phenomenon of ``phase transition''
in $\ell_1$-norm minimization, whereby for fixed $n$ and $m$ and
increasing values of $k$, the probability of success on randomly generated
$k$-sparse $n$-vectors suddenly goes from 100\% to 0\%.
We compare numerical results for Array code binary matrices and DeVore
binary matrices with randomly generated Gaussian matrices, for which
a formal theory is available.

\subsection{Guaranteed Recovery}\label{ssec:num1}

In this subsection, we compare the number of measurements $m$ and
the CPU time for $\ell_1$-norm minimization, when $n = 149^2 = 22,201$,
for two different values of $k$, namely $k = 14$ and $k = 69$.
Note that both values of $k$ are smaller than $\sqrt{n}$.
For each of the array code matrix, the DeVore matrix and a random
Gaussian matrix, the number of measurements $m$ is chosen so as to
\textit{guarantee} robust $k$-sparse recovery using basis pursuit.
In the case of the random Gaussian matrix, the failure probability
$\xi$ is chosen as $10^{-9}$ and the number of samples $m$ is
chosen in accordance with Theorem \ref{thm:11}, specifically \eqref{eq:03}.

When $n = 149^2$ and $k = 14$, with the array code matrix we choose
$q = \sqrt{n} = 149$ and
$d_L = k+1 = 15$, which leads to $m = d_L \sqrt{n} = 2,235$ measurements.
With DeVore's construction, we choose $q$ to be the next largest prime
after $2k$, namely $q = 29$ and $m = 29^2 = 841$.
Because $k < \sqrt{n}/4$, the DeVore construction requires fewer measurements
than the array code matrix, as shown in Example \ref{exam:81}.
When $k = 69$, with the array code matrix we choose $d_L = k+1 = 70$
and $m = d_L \sqrt{n} = 10,430$ measurements.
In contrast, with the DeVore construction, we choose $q$ to be the next
largest prime after $2k$, namely $139$, which leads to $m = q^2 = 19,321$.
Because $k > \sqrt{n}/4$, the DeVore construction requires more measurements
than the array code matrix, as shown in Example \ref{exam:81}.
For the random Gaussian matrix, when $k = 14$, Equation
\eqref{eq:03} gives $m = 11,683$.
% when $n = 149^2, k = 14$, and 
% When $k = 69$,
When $k = 69$, Equation \eqref{eq:03} gives
$m = 44,345$, that is, \textit{more than} $n$.
Therefore using random Gaussian matrices is not meaningful in this case.
% Therefore there was no point in running the Gaussian method with $k = 69$.

The results are shown in Table \ref{table:1}.
From this Table it can be seen that both classes of binary matrices
(DeVore and array code) require significantly less CPU time compared to
random Gaussian matrices.
As shown in Example \ref{exam:81}, the DeVore matrix is to be preferred
when $k < \sqrt{n}/4$ while the array code matrix is to be preferred
when $k > \sqrt{n}/4$.
But in either case, both classes of matrices are preferable to
random Gaussian matrices.

\begin{table}[]
\centering
\caption{Comparison of DeVore, array code and random Gaussian
matrices for $n = 149^2 = 22,201$ and $k = 14, 69$}
\resizebox{0.5\textwidth}{!}{  \begin{tabular}{|c||c|c||c|c||c|c|}

\hline
& \multicolumn{2}{c|}{\textbf{Array Matrix}}
& \multicolumn{2}{c|}{\textbf{DeVore Matrix}} 
& \multicolumn{2}{c|}{\textbf{Gaussian Matrix}}\\
\hline
\textbf{$k$} & \textbf{$m_A$} & \textbf{$T$ in sec.} &
\textbf{$m_D$} & \textbf{$T$ in sec.} &
\textbf{$m_G$} & \textbf{$T$ in sec.}\\
\hline
	14 & 2,235 & 29.014 & 841 & 15.94 & 11,683 & 259,100 \\
69 & 10,430 & 248.5 & 19,321 & 1795 & 44,345 & 692,260 \\
\hline
\end{tabular}}
\label{table:1}
\end{table}

\subsection{Phase Transition Study}\label{ssec:num2}

In this subsection we compare the phase transition behavior of the
basis pursuit formulation with both classes of binary matrices (DeVore
and array code) and random Gaussian matrices.

Suppose we choose integers $n, m < n$, together with a matrix $A$ and
use basis pursuit as the decoder.
If a $k$-sparse vector is chosen at random, we can ask:
What is the probability that $(A,\DBP)$ recovers the vector and how does
it change as $k$ is increased?
We would naturally expect that the probability of success would be 100\%
for $k$ sufficiently small (because various sufficient conditions for
guaranteed recovery would be satisfied) and 0\% for $k$ sufficiently large.
Further, we would expect a gradual drop-off for in-between values of $k$.
The reality however is quite different.
There is a sharp transition between success and failure, which is known
as a phase transition.

To make the discussion precise, let us define two
quantities: $\th := m/n$, which is known as the under-sampling ratio
and $\phi := k/m$, which is known as the oversampling ratio.\footnote{This
terminology is introduced in \cite{Donoho06b} with $m/n$ denoted by $\d$
and $k/m$ denoted by $\r$.
Since these symbols are used to denote different quantities in the
compressed sensing literature, we use $\th$ and $\phi$ instead.}
For fixed $m,n$, let us vary $k$ and make a plot of $\th$ versus $\phi$.
We can compute three quantities:
$\phi_{95}$, which is the value at which the probability of recovering
a random $k$-sparse vector is 95\%, $\phi_{50}$ and $\phi_5$, with
obvious definitions.
The difference $\phi_5 - \phi_{95}$ is called the \textbf{transition width}
and is denoted by $w$.

The phase transition
phenomenon is analyzed theoretically in a series of papers, for
the case, where the measurement matrix $A$ consists of $mn$ independent
samples of Gaussian random variables, using convex polytope theory
\cite{Donoho06b,Donoho-Tanner-PTRSA09}.
A formula is derived for $\phi_{50}$ as a function $\th$,
which might be referred to as the ``transition boundary.''
However, this is not a closed-form formula.
It is further shown that the transition width is roughly equal to $C/\sqrt{n}$,
where $C$ is a constant that does not depend on $n$.
In addition, it is shown through numerical simulations in
\cite{Donoho-Tanner-PTRSA09,Bayati-et-al15,MJGD13}
that a large class of random and deterministic
measurement matrices display the same phase
transition behavior as Gaussian matrices, even though there is as yet
no theoretical analysis for anything other than random Gaussian matrices.

Against this background, it is of interest to study whether the two classes
of binary matrices studied here, namely the array code matrix and the DeVore 
construction, also display the same phase
transition behavior as Gaussian matrices.
Specifically, we study the following questions through numerical simulations:
% For all values of $n$ studied, and across different methods of generating
% measurement matrices (Gaussian, DeVore and array code),
\ben
\item For a given $\th$, is the 50\% recovery value of $\phi_{50}$ 
more or less the same for all three types of matrices?
% \item Is the 95\% recovery value of $\phi$
\item Is the phase transition width $w$ more or less
the same for all three types of matrices?
\item As $n$ is varied, does the phase transition width vary as $C/\sqrt{n}$
for some constant $C$ that is \textit{independent} of the method used
to generate the measurement matrix?
\item What is the CPU time with each type of binary matrix?
\een

Here we give details of the study.
For Gaussian measurement matrices and the DeVore measurement matrices,
the dimension of the vector $n$ is chosen to be $1024$,
to match the previous literature on the topic.
The phase transition boundary for the Gaussian case is computed using
the software provided by Prof.\ David Donoho.
% For the DeVore class of binary matrices, we again chose $n = 1024 = 2^{10}$
% which is an even power of a prime number.
For the array code class, we chose $n = 961 = 31^2$, which is the nearest
square of a prime number to $1024$.
Once $n$ is chosen, for the Gaussian matrices, \textit{every} value of
$m$ (the number of measurements) is permissible.
However, for each class of binary matrix, there are only certain values of $m$
that are permissible.
For the DeVore class, $m$ equals the square of a prime power 
$q$ such that $m = q^2 < n$.
Thus the permissible choices for $q$ are
\bd
\{ 11, 13, 16, 17, 19, 23, 25, 29, 31 \} .
\ed
Note that we omitted the possibility of $q = 8$ as being too small.
In the case of array matrices $n = 31^2 = q^2$ and the permissible values
of $m$ are $lq$ as $l$ ranges from $2$ to $q-1 = 30$, that is,
$\{ 62, 93 , \ldots , 930 \}$.
For each permissible choice of $m$, an appropriate measurement matrix $A$
is generated.
Once this is done, 100 random $k$-sparse vectors are generated
and $\ell_1$-norm minimization (basis pursuit) is applied to each
random $k$-sparse vector with the measurement matrix of each class.
The optimization is carried out using the CVX package of {\tt Matlab}.

Since there is a great deal of information to be presented,
we first show the results for the DeVore construction of \cite{DeVore07}
in Figure \ref{fig:1} and then the results for the array code construction
in Figure \ref{fig:2}.
These figures show $\phi_5, \phi_{50}$ and $\phi_{95}$ for the two methods.

Then in Figure \ref{fig:3}, we plot the numerically determined
median values $\phi_{50}$ for the two classes of binary matrices (DeVore
and array code), together with the theoretically determined values
from \cite[Figure 1]{Donoho-et-al-PNAS09}.\footnote{We
thank Prof.\ David Donoho for providing the software to reproduce the curve.}
Note that there are two theoretical curves here, corresponding to the
case, where the unknown vector $x$ is $k$-sparse with each nonzero value
equal to $\pm 1$ (blue curve) and where each nonzero value is uniformly
distributed over $[-1,1]$ (magenta curve).
The first case is known as ``random signed vector'' and the second
case is known as ``random bounded vector.''
From Figure \ref{fig:3}, it can be seen that the observed transition boundary
in each of the two binary matrices closely matches the theoretical transition
boundary with Gaussian matrices and random signed vectors.
In contrast, the transition boundary value of $\phi$ (at which the success
ratio is 50\%) with random vectors taking arbitrary values in $[-1,1]$
is much lower with Gaussian matrices than with either of the two binary
matrices.

\begin{figure}
\includegraphics [height=6cm, width=9cm]{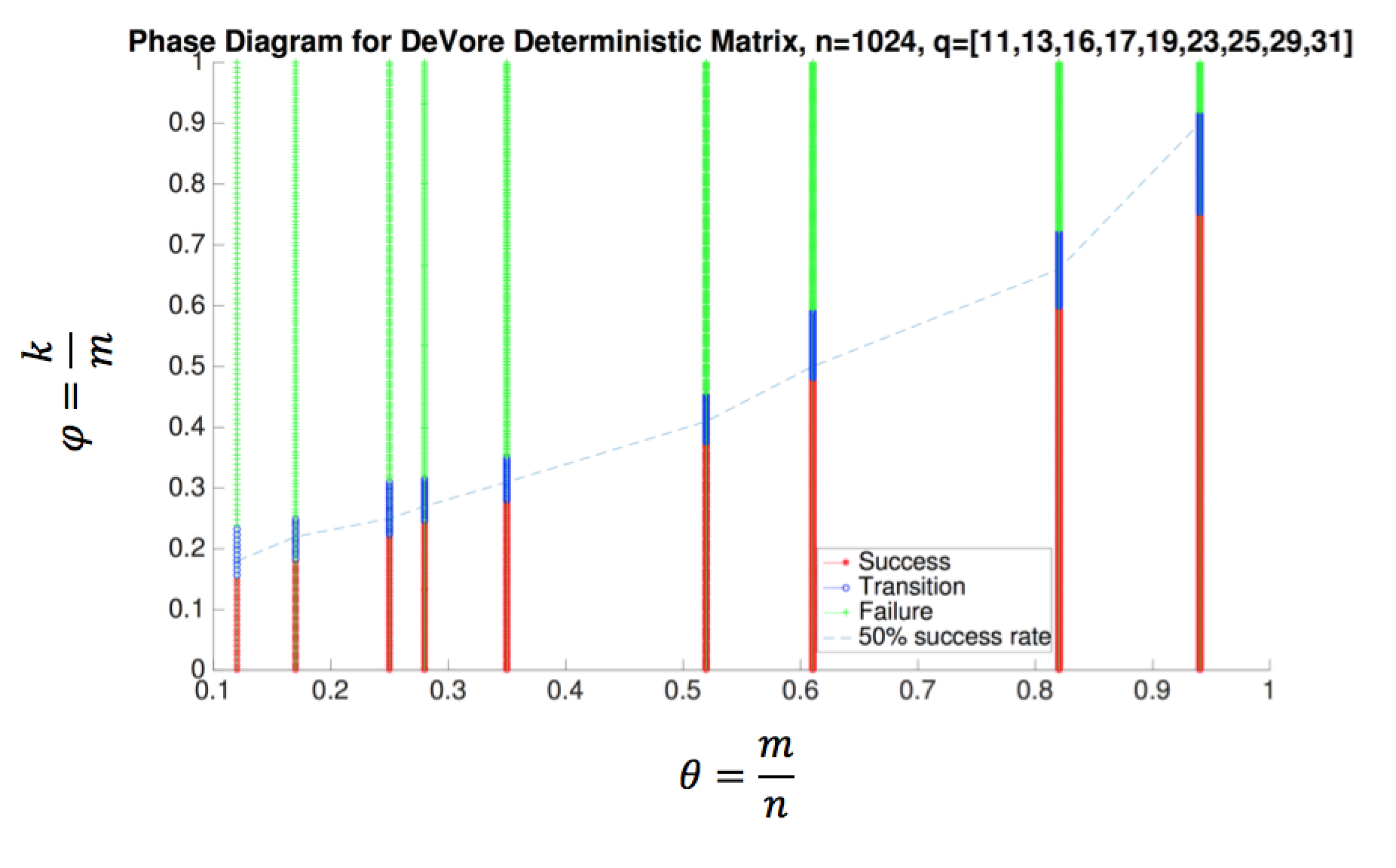}
\caption{Phase transition diagram with success, transition and failure regions for $n=1024$ using DeVore measurement matrix}
\label{fig:1}
\end{figure}

  \begin{figure}
  \includegraphics [height=6cm, width=9cm]{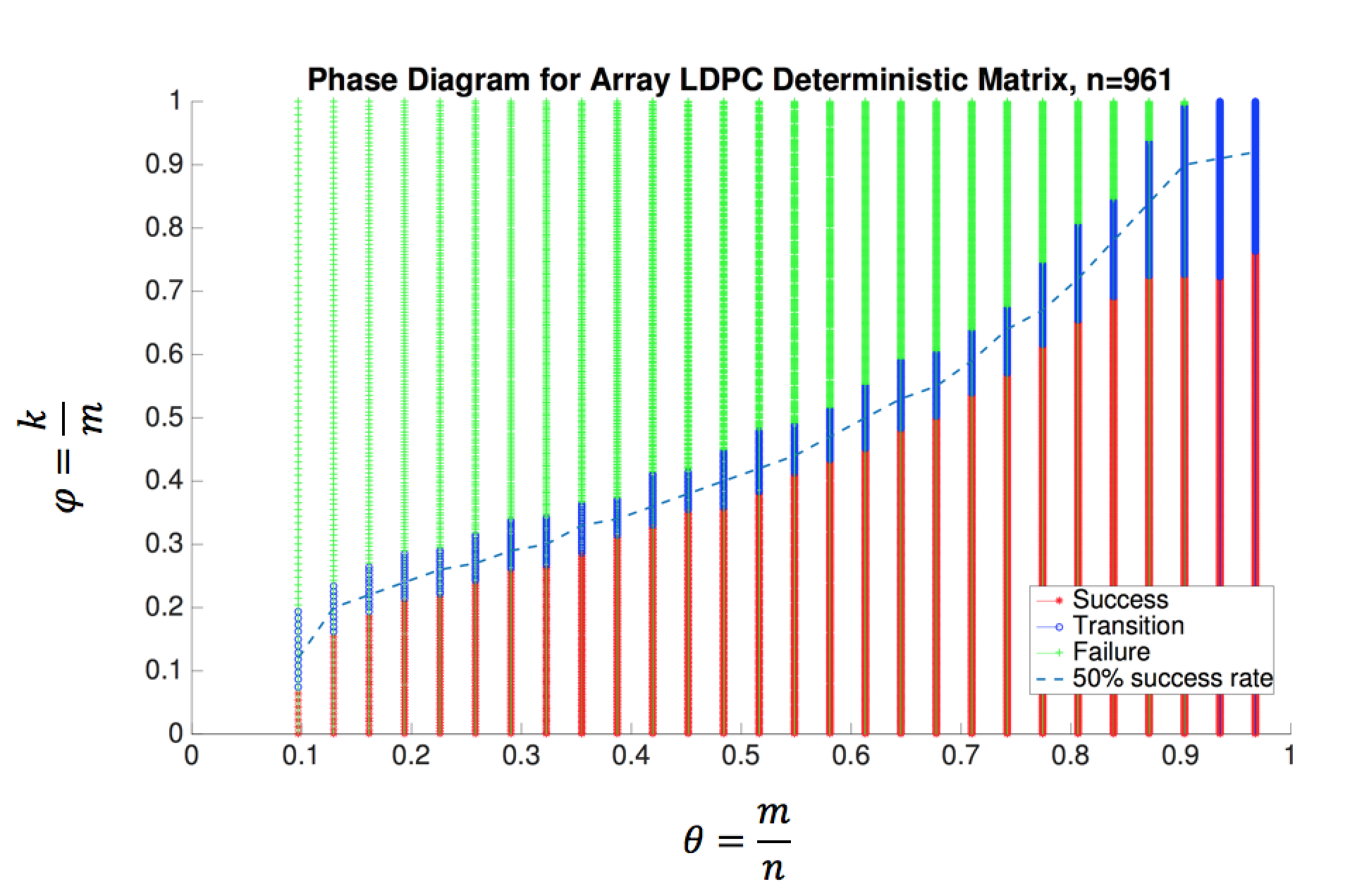}
  \caption{Phase transition diagram with success, transition and failure regions for $n=961$ using array LDPC parity check matrix}
  \label{fig:2}
\end{figure}

\begin{figure}[!htbp]
\includegraphics [height=6cm, width=9cm]{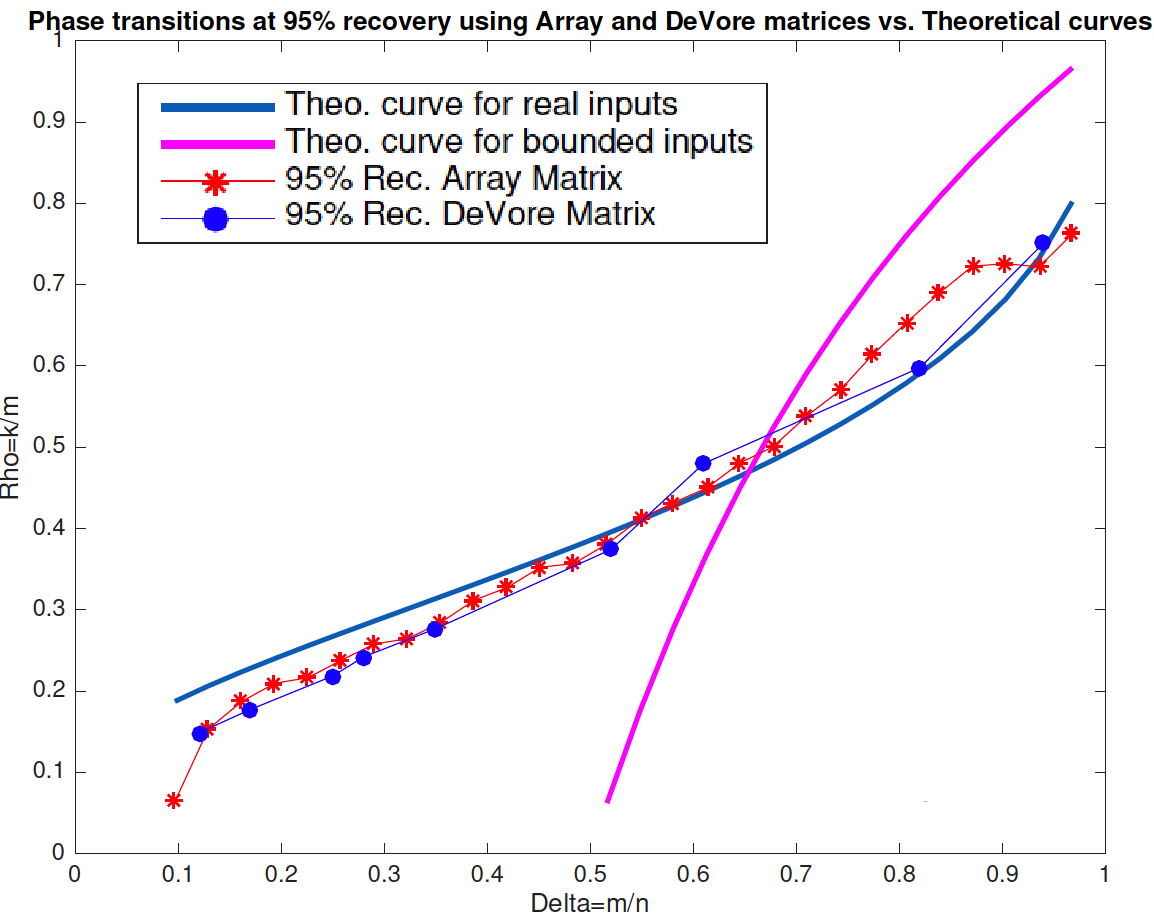}
\caption{Phase transition boundaries for array code binary, DeVore binary
and Gaussian matrices.
For the latter, two boundaries are shown: For signed ($\pm 1$) vectors
and vectors assuming values in $[-1,1]$.
Theoretical curves for real and bounded inputs and 95\% recovery curve using array LDPC parity check matrix and DeVore matrix}
\label{fig:3}
\end{figure}

\setlength\belowcaptionskip{-10pt}

Next we analyze the results shown in these figures.
To make the comparisons between methods readable,
we dispay the results in two separate tables.
Table \ref{table:table1} gives a comparison between the DeVore binary matrices
and random Gaussian matrices.
% where $n$ is the same in both cases.
Table \ref{table:table3} gives a comparison between the array code binary
matrices and random Gaussian matrices.

\begin{table}[]
\centering
\caption{Comparison of transition widths $w$, $50\%$ success rate value
$\phi_{50}$,
% average width $\bar{w}$
and CPU time $T$ for $n=1024$,
using Binary DeVore matrix and Gaussian measurement matrix (subscript $b$ and $g$ respectively) }
\label{table:table1}
\begin{tabular}{|c||c|c||c|c||c|c|}
\hline
$\theta$ & $w_b$  & $w_g$  & $\phi_{50_{b}}$ & $\phi_{50_{g}}$ & $T_b$ in sec. & $T_g$ in sec.  \\ \hline
0.12  & 0.083 & 0.074 & 0.18   & 0.2    & 70   & 182   \\ \hline
0.17  & 0.071 & 0.071 & 0.22   & 0.22   & 106  & 416   \\ \hline
0.25  & 0.09  & 0.078 & 0.25   & 0.27   & 168  & 1435  \\ \hline
0.28  & 0.073 & 0.059 & 0.27   & 0.28   & 222  & 1484  \\ \hline
0.35  & 0.072 & 0.066 & 0.31   & 0.32   & 316  & 5038  \\ \hline
0.52  & 0.08  & 0.07  & 0.41   & 0.39   & 636  & 8695  \\ \hline
0.61  & 0.11  & 0.09  & 0.5    & 0.46   & 1695 & 12810 \\ \hline
0.82  & 0.12  & 0.1   & 0.66   & 0.63   & 1744 & 13453 \\ \hline
0.94  & 0.17  & 0.15  & 0.9    & 0.77   & 2261 & 15827 \\ \hline
% $\bar{w}$     & 0.097 & 0.084 & \multicolumn{4}{c|}{-}         \\ \hline
\end{tabular}
\end{table}

\begin{table}[]
\centering

\caption{Phase transition widths $w$, $50\%$ success rate width $\phi_{50}$, average width $\bar{w}$ for $n=961$ using array LDPC parity check matrix}
\label{table:table3}
\begin{tabular}{|c|c|c|c|c|}
\hline
\textbf{$\theta$} & \textbf{$m$} & \textbf{$w$} & \textbf{$\phi_{50}$} & \textbf{$T$} in sec. \\ \hline
0.1935         & 186        & 0.08       & 0.24         & 0.9423     \\ \hline
0.2258         & 217        & 0.08       & 0.24         & 0.9351     \\ \hline
0.2581         & 248        & 0.08       & 0.27         & 0.8931     \\ \hline
0.2903         & 279        & 0.08       & 0.29         & 0.8732     \\ \hline
0.3548         & 341        & 0.07       & 0.33         & 0.8458     \\ \hline
0.5161         & 496        & 0.1        & 0.42         & 0.6909     \\ \hline
0.6129         & 589        & 0.1        & 0.5          & 0.5946     \\ \hline
0.8387         & 806        & 0.16       & 0.78         & 0.1818     \\ \hline
0.9355         & 899        & 0.28       & 0.91         & 0.0385     \\ \hline
\textbf{$\bar{w}$}     & \multicolumn{4}{c|}{0.1144}                         \\ \hline
\end{tabular}
\end{table}

Next, we compute the transition width ($\phi_5 - \phi_{95}$) for various
values of $\th$, for three different values of $n$ namely $256, 512$
and $1,024$, using the DeVore binary matrix.
The objective is to determine whether the transition width varies as
$C_1/\sqrt{n}$ for some constant $C_1$ that is independent of $n$.
For a fixed choice of $n$, for each (permissible) value of $\th$, we compute
the transition width $w$ and see how constant it is with respect to $\th$.
It can be seen from the table that indeed $w$ is relatively constant even
as $\th$ varies.
Then we averaged the various values of $w$ over $\th$ for each fixed $n$,
to arrive at an average transition width, shown as $\bar{w}$ in
the table.
Then we computed the ratio $\bar{w}/\sqrt{n}$ for the three values of $n$
and called it $C_1$.
The expectation is that this constant $C_1$ should be independent of $n$.
In reality, the values of $C_1$ for $n = 256$ and $512$ are quite close,
while that for $n = 1,024$ is noticeably higher.

\begin{table}[]

\hspace{30em}
\vspace{2em}
\caption{Phase transition widths $w$,
$50\%$ success rate value $\phi_{50}$
% average width $\bar{w}$
and the constant $C_1$ for three different values,
$n=256, 512, 1024$ using DeVore's Binary measurement matrix}
\label{table:table2}

\bc
\resizebox{\columnwidth}{!}{\btab{|c|c|c|c|c|c|c|c|c|c|c|c|}
\hline
$n$                    & $\theta$        & $w$      & $\phi_{50}$      & $n$                    & $\theta$      & $w$       & $\phi_{50}$      & $n$                     & $\theta$       & $w$       & $\phi_{50}$      \\ \hline
\multirow{9}{*}{256} & 0.19     & 0.16   & 0.2    & \multirow{9}{*}{512} & 0.16   & 0.11    & 0.2    & \multirow{9}{*}{1024} & 0.12    & 0.083   & 0.18   \\ \cline{2-4} \cline{6-8} \cline{10-12} 
                     & 0.25     & 0.16   & 0.22   &                      & 0.24   & 0.09    & 0.24   &                       & 0.17    & 0.071   & 0.22   \\ \cline{2-4} \cline{6-8} \cline{10-12} 
                     & 0.32     & 0.14   & 0.31   &                      & 0.33   & 0.095   & 0.3    &                       & 0.25    & 0.09    & 0.25   \\ \cline{2-4} \cline{6-8} \cline{10-12} 
                     & 0.47     & 0.16   & 0.36   &                      & 0.5    & 0.11    & 0.4    &                       & 0.28    & 0.073   & 0.27   \\ \cline{2-4} \cline{6-8} \cline{10-12} 
                     & 0.66     & 0.17   & 0.47   &                      & 0.57   & 0.11    & 0.43   &                       & 0.35    & 0.072   & 0.31   \\ \cline{2-4} \cline{6-8} \cline{10-12} 
                     & -        & -      & -      &                      & 0.71   & 0.15    & 0.52   &                       & 0.52    & 0.08    & 0.41   \\ \cline{2-4} \cline{6-8} \cline{10-12} 
                     & -        & -      & -      &                      & -      & -       & -      &                       & 0.61    & 0.11    & 0.5    \\ \cline{2-4} \cline{6-8} \cline{10-12} 
                     & -        & -      & -      &                      & -      & -       & -      &                       & 0.82    & 0.12    & 0.66   \\ \cline{2-4} \cline{6-8} \cline{10-12} 
                     & -        & -      & -      &                      & -      & -       & -      &                       & 0.94    & 0.17    & 0.9    \\ \hline
$\bar{w}$                   & \multicolumn{3}{c|}{0.16}  & \multirow{3}{*}{}    & \multicolumn{3}{c|}{0.11} & \multirow{3}{*}{}     & \multicolumn{3}{c|}{0.097} \\ \cline{1-4} \cline{6-8} \cline{10-12} 
$C_1$                   & \multicolumn{3}{c|}{2.56} &                      & \multicolumn{3}{c|}{2.53}  &                       & \multicolumn{3}{c|}{3.104} \\ \cline{1-4} \cline{6-8} \cline{10-12} 
% $\Delta$                & \multicolumn{3}{c|}{0.01}  &                      & \multicolumn{3}{c|}{0.009} &                       & \multicolumn{3}{c|}{0.011} \\ \hline
\end{tabular}}
\ec
\end{table}

\section{Discussion}\label{sec:Disc}

In this paper we have built upon previously proven sufficient
conditions for \textit{stable} $k$-sparse recovery and showed that they
actually guarantee \textit{robust} $k$-sparse recovery, that is, enable
basis pursuit to achieve $k$-sparse recovery in the presence of measurement noise.
We then derived a
\textit{universal lower bound} on the number of measurements in order
for binary matrix to satisfy this sufficient condition.
Ideally, we would like to prove a universal \textit{necessary} condition
along the following lines:
If a left-regular binary measurement matrix $A$ achieves robust $k$-sparse
recovery of order $k$,
then $d_L \geq \phi(k)$, where $\phi(\cdot)$ is some function
that is waiting to be discovered.
In such a case, the bounds in Theorem \ref{thm:2} would truly be universal.
At present, there are no known universal necessary conditions for binary
measurement matrices, other than Theorem \ref{thm:12}, which is applicable
to \textit{all} matrices, not just binary matrices.

Note that, as shown in \cite[Problem 13.6]{FR13}, a binary matrix
does not satisfy the RIP of order $k$ with constant $\d$ unless
\bd
m \geq \min \left\{ \frac{1-\d}{1+\d} n , \left( \frac{1-\d}{1+\d} \right)^2
k^2 \right\} .
\ed
This negative result has often been used to suggest that binary matrices
are not suitable for compressed sensing.
However, RIP is only a \textit{sufficient} condition for robust sparse
recovery and as shown here, it is possible to provide far weaker
sufficient conditions for robust sparse recovery in terms of the RNSP,
when the measurement matrix is binary.
This is consistent with the results of \cite{MV-Ranjan19},
which show that RIP implies RNSP.
Hence any sufficient condition that is derived using the RIP can also be 
derived using the RNSP.
The present paper goes farther by deriving a sufficient condition 
based on the RNSP that is strictly weaker than the best available
condition based on the RIP.

Moreover, it is possible to compare the sample complexities implied by
\eqref{eq:03} for random Gaussian matrices with those corresponding
to the DeVore class and the array code class, to see that when $n < 10^5$
and $k < \sqrt{n}$, in fact binary matrices require fewer measurements,
as shown in Table \ref{table:Bounds}.

\begin{table*}[]
\caption{Comparison of the number of measurements for the DeVore binary
	matrix, the array code binary matrix, and the random Gaussian matrix.
	Note that $m_D = q_D^2$ and $m_A = (k+1)q_A$.
	The quantity $m_G$ is computed according to \eqref{eq:03}.}
%\resizebox{2\columnwidth}{!}
	{\begin{tabular}{|c|c|c|c|c|c|c||c|c|c|c|c|c|c||c|c|c|c|c|c|c|}
\hline
\textbf{$n$}           & \textbf{$k$} & \textbf{$q_D$} & \textbf{$m_D$} & \textbf{$q_A$}         & \textbf{$m_A$} & \textbf{$m_G$} & \textbf{$n$}                             & \textbf{$k$} & \textbf{$q_D$} & \textbf{$m_D$} & \textbf{$q_A$}          & \textbf{$m_A$} & \textbf{$m_G$} & \textbf{$n$}                             & \textbf{$k$} & \textbf{$q_D$} & \textbf{$m_D$} & \textbf{$q_A$}          & \textbf{$m_A$} & \textbf{$m_G$} \\ \hline
\multirow{4}{*}{900} & 5          & 11          & 121         & \multirow{4}{*}{31} & 186         & 4,467       & \multirow{4}{*}{$10^4$} & 20         & 47          & 2,209       & \multirow{4}{*}{101} & 2,121       & 14,436      & \multirow{4}{*}{$10^5$} & 50         & 101         & 10,201      & \multirow{4}{*}{317} & 16,167      & 39,165      \\ \cline{2-4} \cline{6-7} \cline{9-11} \cline{13-14} \cline{16-18} \cline{20-21} 
                     & 10         & 23          & 529         &                     & 341         & 6,682       &                                        & 40         & 83          & 6,889       &                      & 4,141       & 25,430      &                                        & 100        & 211         & 44,521      &                      & 32,017      & 71,878      \\ \cline{2-4} \cline{6-7} \cline{9-11} \cline{13-14} \cline{16-18} \cline{20-21} 
                     & 15         & 31          & 961         &                     & 496         & 8,982       &                                        & 60         & 127         & 16,129      &                      & 6,161       & 35,600      &                                        & 150        & 307         & 94,249      &                      & 47,867      & 102,604     \\ \cline{2-4} \cline{6-7} \cline{9-11} \cline{13-14} \cline{16-18} \cline{20-21} 
                     & 20         & 41          & 1,681       &                     & 651         & 10,863      &                                        & 80         & 163         & 26,569      &                      & 8,181       & 45,232      &                                        & 200        & 401         & 160,801     &                      & 63,717      & 132,030     \\ \hline
\end{tabular}}
\label{table:Bounds}
\end{table*}

One might argue that the bound in \eqref{eq:03} is only a \textit{sufficient}
condition for the number of measurements and that in actual examples,
far fewer measurements suffice.
This is precisely the motivation behind studying the phase transition
of basis pursuit with binary matrices.
As shown in Section \ref{ssec:num2}, in fact there is no difference
between the phase transition behavior of random Gaussian matrices
and binary matrices.
This observation reinforces earlier observations in \cite{MJGD13}.
In other words, the fraction of randomly generated $k$-sparse
vectors that can be recovered using $m$ measurements is the same
whether one uses Gaussian matrices or binary matrices.
Given that basis pursuit can be implemented much more efficiently
with binary measurement matrices than with random Gaussian matrices
and both classes of matrices exhibit similar phase transition properties,
there appears to be a very strong case for preferring binary measurement
matrices over random Gaussian matrices, notwithstanding the ``order-optimality''
of the latter class.
In this connection, it would be worthwhile to explore whether
other classes of measurements also exhibit phase transition behavior
that is quantitatively similar to that of Gaussian and binary matrices.

There is one final point that we wish to make.
Theorem \ref{thm:3} suggests that, in order to use binary matrices for
compressed sensing, it is better to use graphs with \textit{small} girth,
in fact, of girth six.
This runs counter to the intuition in LDPC decoding, where one wishes
to design binary matrices with \textit{large} girth.
Indeed, in \cite{KSDH11}, the authors build on an earlier paper
\cite{Arora-et-al09} and develop a message-passing type of decoder
that achieves order-optimality using a binary matrix.
The binary matrices that are used in \cite{KSDH11} all have \textit{large}
girth $\OM(\ln n)$, which is the theoretical upper bound.
% This discrepancy needs to be explored further.
One possible explanation for this discrepancy
% given in Section \ref{sec:low},
is that the model for compressed sensing
using in \cite{KSDH11} is different from the one used here and in most of
the compressed sensing literature.
Specifically (to paraphrase a little bit), in \cite{KSDH11} in the unknown
vector, each component is binary and the probability that the component 
equals one is $k/n$.
Thus, the \textit{expected value} of nonzero bits is $k$, but it could
be larger or smaller.
Accordingly, the actual sparsity count is a random number that could exceed $k$.
The recovery results proved in \cite{KSDH11} are also probabilistic in nature.
It is worth further study to determine whether this difference is sufficient
to explain why, in compressed sensing, graphs of low girth are to be preferred.

\section*{Acknowledgement}

The authors thank Prof.\ David Donoho of Stanford
University for his helpful suggestions on phase transitions and
for providing the code to enable us to reproduce his computational results.
They also thank Prof.\ Phanindra Jampana of IIT Hyderabad for helpful
discussions on the construction of Euler squares.
Finally, they thank the reviewers for their careful reading of the previous
draft and for detailed comments that have greatly improved the readability
of the paper.

\bibliographystyle{IEEEtran}
\bibliography{Comp-Sens}

\end{document}